\documentclass[11pt]{article}

\usepackage[margin=1in]{geometry}
\usepackage{amsmath,amssymb,amsthm,mathtools,bm}
\usepackage[T1]{fontenc}
\usepackage{lmodern}
\usepackage{microtype}
\newif\ifTutorialNatbib
\IfFileExists{natbib.sty}{\TutorialNatbibtrue}{}
\ifTutorialNatbib
  \usepackage[authoryear,round]{natbib}
  \setcitestyle{citesep={;},aysep={,},yysep={;}}
  \let\cite\citep
\else
  \usepackage{etoolbox}
\fi
\usepackage{xcolor}
\usepackage{hyperref}
\definecolor{referenceblue}{rgb}{0,0.08,0.45}
\ifTutorialNatbib\else
  \makeatletter
  \renewcommand{\@cite}[2]{(#1\if@tempswa, #2\fi)}
  \renewcommand{\@cite@ofmt}[1]{#1}
  \patchcmd{\@citex}{,\penalty\@m\ }{;\penalty\@m\ }{}{}
  \newcommand{\tutorialBibitem}[2][]{\tutorialBibitemLabel{#2}#1\@nil}
  \def\tutorialBibitemLabel#1#2(#3)#4\@nil{%
    \tutorialOriginalBibitem[#2, #3]{#1}}
  
  \makeatother
\fi
\DeclareFontFamily{T1}{tutorialmono}{}
\DeclareFontShape{T1}{tutorialmono}{m}{n}{<-> ec-lmtt12}{}
\usepackage{enumitem}
\usepackage{graphicx}
\usepackage{booktabs,tabularx}

\hypersetup{
  colorlinks=true,
  linkcolor=blue,
  citecolor=blue,
  urlcolor=referenceblue
}

\newtheorem{theorem}{Theorem}[section]
\newtheorem{proposition}[theorem]{Proposition}
\newtheorem{lemma}[theorem]{Lemma}
\newtheorem{corollary}[theorem]{Corollary}
\newtheorem{definition}[theorem]{Definition}
\theoremstyle{remark}

\newcommand{\R}{\mathbb{R}}
\newcommand{\E}{\mathbb{E}}
\newcommand{\N}{\mathcal{N}}
\newcommand{\dd}{\,\mathrm{d}}
\newcommand{\grad}{\nabla}
\newcommand{\diver}{\nabla\!\cdot}
\newcommand{\lap}{\Delta}

\title{A Tutorial on Diffusion Theory: From Differential Equations to Diffusion Models}
\author{
Jiayi Fu, \qquad Yuxia Wang \\
INSAIT, Sofia University ``St.\ Kliment Ohridski'', Bulgaria\\
\quad
\texttt{jiayi.fu@insait.ai}, \qquad
\texttt{yuxia.wang@insait.ai}
}
\date{}

\begin{document}

\maketitle

\setlength{\emergencystretch}{1em}

\begin{abstract}
Diffusion models have emerged as a leading framework for generative modeling, yet their mathematical foundations are often presented in separate treatments. This tutorial offers a unified, self-contained introduction organized into three parts. First, we construct a forward generative stochastic differential equation (SDE) that transports noise to data, marginalize over the clean data endpoint, and derive training objectives directly from the forward dynamics. Second, we begin with a conditional path from data to noise, marginalize it to obtain the forward process, and derive the corresponding reverse-time SDE and training objectives. Within this forward–reverse framework, we also establish the connection between denoising diffusion probabilistic models (DDPMs) and variance-preserving (VP) SDEs. Third, we review numerical methods for solving SDEs and discuss two approaches to guided generation: classifier guidance and classifier-free guidance.
\end{abstract}

\tableofcontents

\clearpage
\part{A Forward Process from Gaussian Noise to Data}\label{part:generative}

\section{Setup}
\label{sec:direct-generative}
\label{sec:part1-setup}

In this part, the sampling clock \(t\in[0,1]\) runs from noise to
data. Draw \(X_0\sim p_0=\N(0,I)\) and
\(X_1\sim p_1=p_{\mathrm{data}}\) independently in \(\R^d\). Let
\(\alpha_t\) and \(\sigma_t\) be nonnegative schedules, continuous
on \([0,1]\) and differentiable on \((0,1)\), with
\[
\alpha_0=0,\qquad \sigma_0=1,\qquad
\alpha_1=1,\qquad \sigma_1=0.
\]
Typically, \(\alpha_t\) increases while \(\sigma_t\) decreases.
Define the interpolated random variable \(X_t\) and its law by
\[
X_t:=\alpha_t X_1+\sigma_t X_0,
\qquad
p_t:=\operatorname{Law}(X_t).
\]
The family \(\{p_t\}_{t\in[0,1]}\) is the \emph{probability path}.
It specifies the distribution at each sampling time. By continuity
of the schedules, \(p_t\) converges weakly to the Gaussian noise law
as \(t\downarrow0\) and to the data law as \(t\uparrow1\).
For \(\sigma_t>0\), conditioning on the clean endpoint \(X_1=x_1\)
gives the Gaussian density
\[
p_t(x\mid x_1)=\N(x;\alpha_t x_1,\sigma_t^2 I).
\]
Averaging this conditional density over the data law \(p_1\) gives
the marginal density
\[
p_t(x)=\int_{\R^d}p_t(x\mid x_1)\,p_1(x_1)\dd x_1.
\]
At \(t=1\), the conditional law given \(X_1=x_1\) is a point mass at
\(x_1\). Sections~\ref{sec:part1-forward-conditional}
and~\ref{sec:part1-forward-marginal} develop the conditional Gaussian
and the marginal path. Wherever the marginal density is positive
and differentiable, its \emph{score} is
\(s_t(x):=\grad\log p_t(x)\). All gradients, divergences, and
Laplacians act on the state variable \(x\).
For the score-parameterized sampling SDE, define the coefficients
at interior times with \(\alpha_t>0\) by
\begin{equation}
f_t:=\frac{\dot{\alpha}_t}{\alpha_t},
\qquad
g_t^2:=-\frac{\dd}{\dd t}\sigma_t^2
       +2\frac{\dot{\alpha}_t}{\alpha_t}\sigma_t^2.
\label{eq:section8-generative-fg}
\end{equation}
We assume \(g_t^2>0\) wherever these coefficients are used and take
\(g_t=\sqrt{g_t^2}\). The SDE runs in increasing \(t\), with
\(W_t\) denoting its standard Brownian motion. Formulas involving
\(\sigma_t^{-1}\) apply only where \(\sigma_t>0\). If an endpoint
is singular, we run the SDE on a truncated interval
\([a,b]\subset(0,1)\) where its coefficients are defined, starting
from the marginal law \(p_a\).

The notation \(\alpha_t,\sigma_t,f_t,g_t\) is local to this part.
Part~\ref{part:reverse} defines these quantities anew for the
data-to-noise clock.

\begin{figure}[tb]
    \centering
    \includegraphics[width=\linewidth]{
        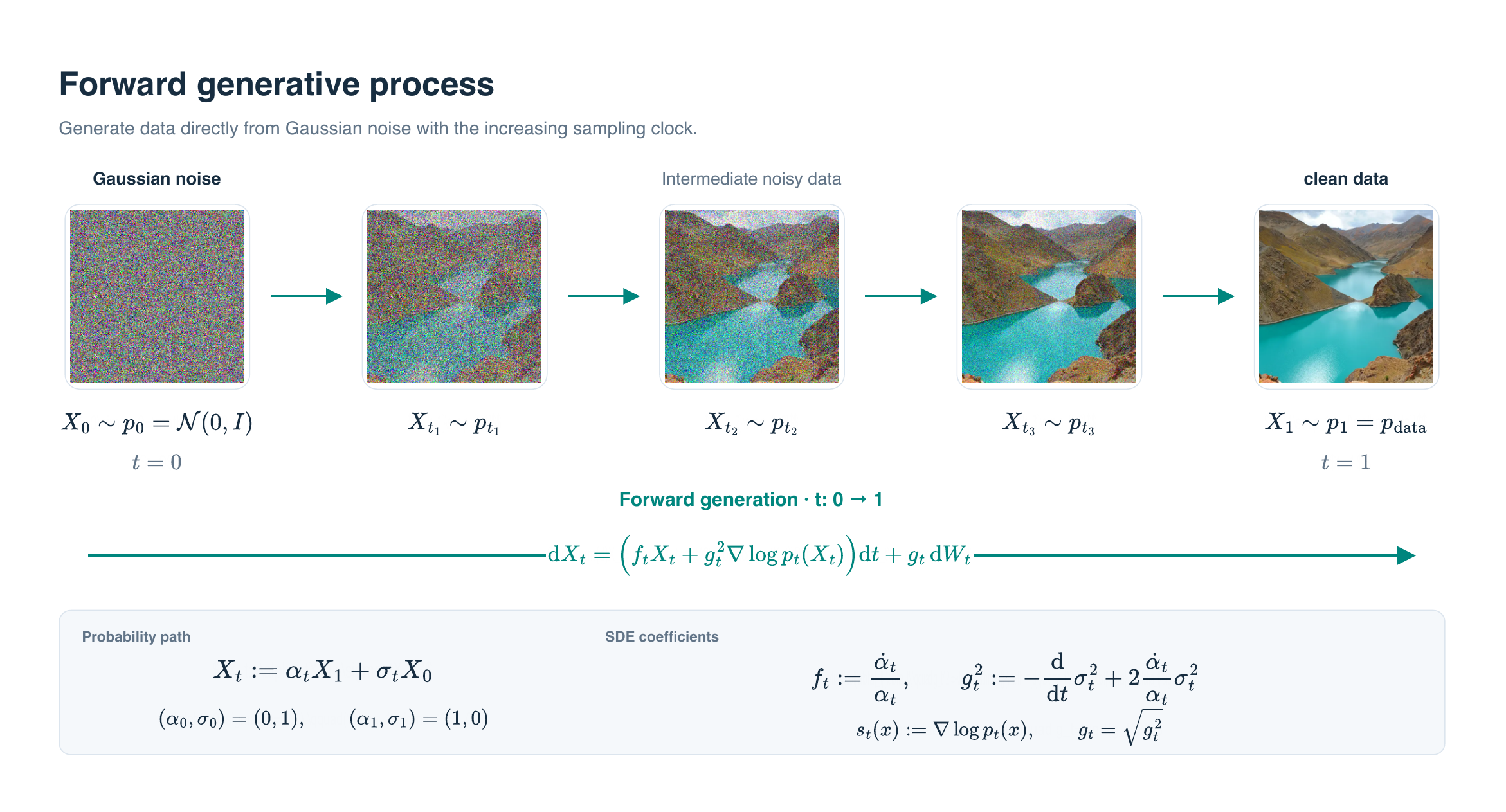
    }
    \caption{
        Forward generative process from Gaussian noise to data.
        The sampling clock increases from $t=0$ to $t=1$,
        transforming $X_0\sim p_0=\mathcal{N}(0,I)$ into
        $X_1\sim p_1=p_{\mathrm{data}}$.
        The arrow indicates the process direction and its governing SDE.
    }
    \label{fig:forward-generative}
\end{figure}

\section{Conditional Forward Process}
\label{sec:part1-forward-conditional}

\subsection{Conditional Gaussian forward path}

Fix a target data point $x_1$. The conditional generative path is
\begin{equation}
p_t(x\mid x_1)
:=
\N(x\mid \alpha_t x_1,\sigma_t^2I).
\label{eq:section8-conditional-gen-kernel}
\end{equation}
At $t=0$, this distribution is Gaussian noise; at $t=1$, it collapses to the clean data point $x_1$.

Its conditional score is
\begin{equation}
\grad\log p_t(x\mid x_1)
=
-\frac{x-\alpha_t x_1}{\sigma_t^2}.
\label{eq:section8-conditional-gen-score}
\end{equation}

\subsection{Conditional forward SDE}

\begin{theorem}[Conditional generative SDE]
For every fixed $x_1$, the conditional path \eqref{eq:section8-conditional-gen-kernel} is generated by the SDE
\begin{equation}
\dd X_t
=
\Bigl(
f_tX_t+g_t^2\grad\log p_t(X_t\mid x_1)
\Bigr)\dd t
+g_t\dd W_t,\quad X_0\sim p_0
\label{eq:section8-conditional-gen-sde}
\end{equation}
\end{theorem}

\begin{proof}
According to the Fokker-Planck Equation (See Appendix~\ref{app:fokker-planck}), this is equivalent to proving that:
$$\partial_t p_t(x\mid x_1)=-\diver \left[ (f_tx+g_t^2\grad\log p_t(x\mid x_1))\cdot p_t(x|x_1)\right]+\frac{1}{2}g_t^2\lap p_t(x\mid x_1)$$

Write
\[
r_t(x):=x-\alpha_t x_1.
\]
Since
\[
p_t(x\mid x_1)
=
\frac{1}{(2\pi \sigma_t^2)^{d/2}}
\exp\!\left(-\frac{\|r_t(x)\|_2^2}{2\sigma_t^2}\right),
\]
\[
\log p_t(x\mid x_1)
=
-\frac{d}{2}\log(2\pi \sigma_t^2)-\frac{\|r_t(x)\|_2^2}{2\sigma_t^2}.
\]
Because
\[
\partial_t r_t(x)=-\dot{\alpha}_t x_1,
\qquad
\partial_t\|r_t(x)\|_2^2=-2\dot{\alpha}_t r_t(x)^\top x_1,
\]
we obtain
\begin{align*}
\partial_t \log p_t(x\mid x_1)
&=
-d\frac{\dot{\sigma}_t}{\sigma_t}
-\partial_t\!\left(\frac{\|r_t(x)\|_2^2}{2\sigma_t^2}\right)\\
&=
-d\frac{\dot{\sigma}_t}{\sigma_t}
+\frac{\dot{\alpha}_t}{\sigma_t^2}r_t(x)^\top x_1
+\frac{\dot{\sigma}_t}{\sigma_t^3}\|r_t(x)\|_2^2.
\end{align*}
Hence
\begin{equation}
\partial_t p_t(x\mid x_1)
=
p_t(x\mid x_1)\left[
-d\frac{\dot{\sigma}_t}{\sigma_t}
+ \frac{\dot{\alpha}_t}{\sigma_t^2}r_t(x)^\top x_1
+ \frac{\dot{\sigma}_t}{\sigma_t^3}\|r_t(x)\|_2^2
\right].
\label{eq:section8-conditional-gen-sde-lhs}
\end{equation}
Next,
\[
\grad \log p_t(x\mid x_1)
=
-\frac{r_t(x)}{\sigma_t^2},
\qquad
\grad p_t(x\mid x_1)
=
-p_t(x\mid x_1)\frac{r_t(x)}{\sigma_t^2}.
\]
Differentiating once more yields
\[
\lap p_t(x\mid x_1)
=
p_t(x\mid x_1)\left(
\frac{\|r_t(x)\|_2^2}{\sigma_t^4}-\frac{d}{\sigma_t^2}
\right),
\]
and
\[
-\diver\bigl(f_t x\,p_t(x\mid x_1)\bigr)
=
p_t(x\mid x_1)\left[
-f_t d + f_t\frac{x^\top r_t(x)}{\sigma_t^2}
\right].
\]
Because $p_t(x\mid x_1)\grad\log p_t(x\mid x_1)=\grad p_t(x\mid x_1)$, the right hand side of Fokker--Planck equation of \eqref{eq:section8-conditional-gen-sde} becomes
\[
-\diver \left[ (f_tx+g_t^2\grad\log p_t(x\mid x_1))\cdot p_t(x|x_1)\right]+\frac{1}{2}g_t^2\lap p_t(x\mid x_1)=
-\diver\bigl(f_t x\,p_t(x\mid x_1)\bigr)-\frac{1}{2}g_t^2\lap p_t(x\mid x_1)
\]
Substituting the previous expressions, the right-hand side of the Fokker--Planck equation becomes
\[
p_t(x\mid x_1)\Biggl[
-f_t d + f_t\frac{x^\top r_t(x)}{\sigma_t^2}
- \frac{g_t^2}{2}\left(
\frac{\|r_t(x)\|_2^2}{\sigma_t^4}-\frac{d}{\sigma_t^2}
\right)
\Biggr].
\]

Now use $x=r_t(x)+\alpha_t x_1$ together with
\[
\frac{g_t^2}{2\sigma_t^2}
=
-\frac{\dot{\sigma}_t}{\sigma_t}+f_t,
\]
which follows from \eqref{eq:section8-generative-fg}. Then the right-hand side simplifies to
\[
p_t(x\mid x_1)\left[
-d\frac{\dot{\sigma}_t}{\sigma_t}
+ \frac{\dot{\alpha}_t}{\sigma_t^2}r_t(x)^\top x_1
+ \frac{\dot{\sigma}_t}{\sigma_t^3}\|r_t(x)\|_2^2
\right],
\]
which agrees with \eqref{eq:section8-conditional-gen-sde-lhs}. Therefore, the density path of \eqref{eq:section8-conditional-gen-sde} is precisely \eqref{eq:section8-conditional-gen-kernel}.
\end{proof}

\section{Marginalized Forward Process}
\label{sec:part1-forward-marginal}

\subsection{Marginalized forward path}
Marginalizing \eqref{eq:section8-conditional-gen-kernel} over the data distribution $p_1$ gives the noise-to-data density path
\begin{equation}
p_t(x)
:=
\int_{\R^d} p_t(x\mid x_1)p_1(x_1)\dd x_1.
\label{eq:section8-marginal-gen-density}
\end{equation}

\subsection{Marginal forward SDE}
\begin{proposition}[Marginal generative SDE]
The marginal density path $p_t$ is generated by
\begin{equation}
\dd X_t
=
\Bigl(
f_tX_t+g_t^2\grad\log p_t(X_t)
\Bigr)\dd t
+g_t\dd W_t,
\qquad
X_0\sim p_0.
\label{eq:section8-marginal-gen-sde}
\end{equation}
\end{proposition}

\begin{proof}
According to the Fokker-Planck Equation (See Appendix~\ref{app:fokker-planck}), this is equivalent to proving that:
$$\partial_t p_t(x)=-\diver \left[ (f_tx+g_t^2\grad\log p_t(x))\cdot p_t(x)\right]+\frac{1}{2}g_t^2\lap p_t(x)$$

From the section~\ref{sec:part1-forward-conditional}, for every fixed $x_1$,  the following conditional Fokker--Planck equation is established:
\[
\partial_t p_t(x\mid x_1)
=
-\diver\Bigl(
\bigl[f_t x+g_t^2\grad\log p_t(x\mid x_1)\bigr]p_t(x\mid x_1)
\Bigr)
+\frac{1}{2}g_t^2\lap p_t(x\mid x_1).
\]
Integrating both sides against $p_1(x_1)\dd x_1$ gives
\begin{align*}
\partial_t p_t(x)
&=
-\int
\diver\Bigl(
\bigl[f_t x+g_t^2\grad\log p_t(x\mid x_1)\bigr]p_t(x\mid x_1)
\Bigr)p_1(x_1)\dd x_1\\
&\qquad
+\frac{1}{2}g_t^2\int \lap p_t(x\mid x_1)p_1(x_1)\dd x_1.
\end{align*}
Since the differential operators act on $x$, we may move them outside the integral:
\begin{align*}
\partial_t p_t(x)
&=
-\diver\left(
\int
\bigl[f_t x+g_t^2\grad\log p_t(x\mid x_1)\bigr]
p_t(x\mid x_1)p_1(x_1)\dd x_1
\right)\\
&\qquad
+\frac{1}{2}g_t^2\lap\left(
\int p_t(x\mid x_1)p_1(x_1)\dd x_1
\right).
\end{align*}
Using \eqref{eq:section8-marginal-gen-density}, this becomes
\begin{align*}
\partial_t p_t(x)
&=
-\diver\left(
f_t x\,p_t(x)
+g_t^2\int
p_t(x\mid x_1)\grad\log p_t(x\mid x_1)p_1(x_1)\dd x_1
\right)
+\frac{1}{2}g_t^2\lap p_t(x).
\end{align*}
Now
\[
p_t(x\mid x_1)\grad\log p_t(x\mid x_1)
=
\grad p_t(x\mid x_1),
\]
so
\[
\int
p_t(x\mid x_1)\grad\log p_t(x\mid x_1)p_1(x_1)\dd x_1
=
\int \grad p_t(x\mid x_1)p_1(x_1)\dd x_1
=
\grad p_t(x).
\]

Substituting this back gives
\[
\partial_t p_t(x)
=
-\diver\Bigl(
\bigl[f_t x+g_t^2\grad\log p_t(x)\bigr]p_t(x)
\Bigr)
+\frac{1}{2}g_t^2\lap p_t(x).
\]
This is exactly the Fokker--Planck equation of \eqref{eq:section8-marginal-gen-sde}, so the SDE \eqref{eq:section8-marginal-gen-sde} generates the marginal density path.
\end{proof}

\section{Learning the Score}
\label{sec:part1-score-learning}

\subsection{The unknown marginal score}
In the marginal forward SDE,
\[
\dd X_t=\left[f_tX_t+g_t^2\grad\log p_t(X_t)\right]\dd t
+g_t\dd W_t,\qquad X_0\sim p_0.
\]
The schedules determine $f_t$ and $g_t$, so the unknown field to be learned is precisely the marginal score $s_t(x)=\grad\log p_t(x)$. 

\subsection{Marginal score matching}
A natural strategy is to approximate the marginal score with a neural network,
\[
s_\theta(x,t)\approx s_t(x)=\grad\log p_t(x).
\]
If the marginal score $s_t(x)=\grad\log p_t(x)$ were directly available, one could train a score model $s_\theta(x,t)$ by
\begin{equation}
\mathcal L_{\mathrm{SM}}(\theta)
:=
\frac{1}{2}\int_0^1 \lambda(t)\,
\E_{X_t\sim p_t}\left[
\|s_\theta(X_t,t)-\grad\log p_t(X_t)\|_2^2
\right]\dd t.
\label{eq:section8-generative-sm}
\end{equation}
Here $\lambda(t)\ge0$ is a time-weighting function, chosen so that the objective is finite. In practice, the marginal score is unavailable, so this objective needs a tractable conditional counterpart.

\subsection{Conditional score matching}
Draw $X_1\sim p_1$ and $X_t\sim p_t(\cdot\mid X_1)$. The Gaussian conditional path has the known score
\[
\grad_x\log p_t(x\mid x_1)
=-\frac{x-\alpha_t x_1}{\sigma_t^2},
\]
as derived in \eqref{eq:section8-conditional-gen-score}. Thus one can train with the conditional score itself as the target:
\begin{equation}
\mathcal L_{\mathrm{CSM}}(\theta)
:=
\frac{1}{2}\int_0^1 \lambda(t)\,
\E_{X_1,\;X_t\sim p_t(\cdot\mid X_1)}\left[
\|s_\theta(X_t,t)-\grad\log p_t(X_t\mid X_1)\|_2^2
\right]\dd t.
\label{eq:section8-generative-csm}
\end{equation}

\begin{theorem}[Conditional score matching is a proxy for marginal score matching]
The objectives \eqref{eq:section8-generative-sm} and \eqref{eq:section8-generative-csm} differ only by a constant independent of $\theta$. Equivalently,
\[
\mathcal L_{\mathrm{CSM}}(\theta)
=
\mathcal L_{\mathrm{SM}}(\theta)+C.
\]
\end{theorem}

\begin{proof}
Fix $t$ and define
\[
U:=X_t,
\qquad
V:=\grad\log p_t(X_t\mid X_1),
\qquad
f(U):=s_\theta(X_t,t).
\]
By Fisher's identity (Appendix~\ref{app:fisher}),
\[
\E[V\mid U]
=
\grad\log p_t(U).
\]
Applying the orthogonality identity from Appendix~\ref{app:orthogonality} gives
\[
\E\|f(U)-V\|_2^2
=
\E\|f(U)-\E[V\mid U]\|_2^2
+ \E\|V-\E[V\mid U]\|_2^2.
\]
Hence
\[
\E\|s_\theta(X_t,t)-\grad\log p_t(X_t\mid X_1)\|_2^2
=
\E\|s_\theta(X_t,t)-\grad\log p_t(X_t)\|_2^2
+ C_t,
\]
where $C_t$ does not depend on $\theta$. Integrating over $t$ yields the result.
\end{proof}

\subsection{The learned sampling SDE}
Once the score has been learned, the forward generative SDE is
\begin{equation}
\dd X_t
=
\Bigl(
f_tX_t+g_t^2 s_\theta(X_t,t)
\Bigr)\dd t
+g_t\dd W_t,
\qquad
X_0\sim p_0.
\label{eq:section8-learned-generative-sde-score}
\end{equation}
With the exact marginal score and the prescribed initialization, this SDE follows the forward noise-to-data marginal path. The learned score approximates this sampling dynamics directly.

\section{Prediction Targets and Loss Functions}
\label{sec:part1-prediction-losses}

The preceding section learns the marginal score directly. We now introduce alternative regression targets, connecting each resulting predictor to the score.
\subsection{Alternative training objectives}
\label{sec:part1-five-losses}
We list all five possible regression targets here,

$$
X_1,\quad X_0,\quad
S_t:=-\frac{X_0}{\sigma_t},\quad
U_t:=\dot\alpha_tX_1+\dot\sigma_tX_0,\quad
V_t:=\alpha_tX_0-\sigma_tX_1.
$$

Here, $X_1$ is the clean data; $X_0$ is the noise; $S_t$ is the conditional score, $S_t=\grad\log p_t(X_t\mid x_1)$; $U_t$ is the velocity, $U_t=\frac{\dd X_t}{\dd t}$; $V_t$ is another velocity mostly used in Variance-Preserving diffusion.
Approximating the above targets using networks $x_{1,\theta},\ x_{0,\theta},\ s_\theta,\ u_\theta,\ v_\theta$, and define the following five losses:
\begin{align}
\mathcal L_{x_1}(\theta)
&:=\frac12\int_0^1\omega_{x_1}(t)\,
\E\|x_{1,\theta}(X_t,t)-X_1\|_2^2\dd t,
\label{eq:part1-prediction-loss-x}\\
\mathcal L_{x_0}(\theta)
&:=\frac12\int_0^1\omega_{x_0}(t)\,
\E\|x_{0,\theta}(X_t,t)-X_0\|_2^2\dd t,
\label{eq:part1-prediction-loss-e}\\
\mathcal L_s(\theta)
&:=\frac12\int_0^1\omega_s(t)\,
\E\|s_\theta(X_t,t)-S_t\|_2^2\dd t,
\label{eq:part1-prediction-loss-s}\\
\mathcal L_u(\theta)
&:=\frac12\int_0^1\omega_u(t)\,
\E\|u_\theta(X_t,t)-U_t\|_2^2\dd t,
\label{eq:part1-prediction-loss-u}\\
\mathcal L_v(\theta)
&:=\frac12\int_0^1\omega_v(t)\,
\E\|v_\theta(X_t,t)-V_t\|_2^2\dd t.
\label{eq:part1-prediction-loss-v}
\end{align}
The expectations are over $X_1$ and independent Gaussian noise. 
\subsection{Connections to score prediction}

Define
\[
\Delta_t:=\alpha_t\dot\sigma_t-\sigma_t\dot\alpha_t,
\qquad R_t:=\alpha_t^2+\sigma_t^2.
\]
Each alternative output supplies the score directly:
\begin{equation}
\begin{aligned}
s_\theta^{(x_1)}(X_t,t)
&:=\frac{\alpha_tx_{1,\theta}(X_t,t)-X_t}{\sigma_t^2},&
s_\theta^{(x_0)}(X_t,t)
&:=-\frac{x_{0,\theta}(X_t,t)}{\sigma_t},\\
s_\theta^{(u)}(X_t,t)
&:=\frac{\dot\alpha_tX_t-\alpha_tu_\theta(X_t,t)}{\sigma_t\Delta_t},&
s_\theta^{(v)}(X_t,t)
&:=-\frac{\sigma_tX_t+\alpha_tv_\theta(X_t,t)}{\sigma_tR_t}.
\end{aligned}
\label{eq:part1-prediction-conversions}
\end{equation}
The $U$ conversion requires $\Delta_t\ne0$; the $V$ conversion requires $R_t>0$. 

\paragraph{Comparing the objectives with score prediction.}
Let $\ell_s:=\|s_\theta(X_t,t)-S_t\|_2^2$, and let
$\ell_{x_1},\ell_{x_0},\ell_u,\ell_v$ be the squared errors in the other objectives. For mutually converted outputs, the samplewise relations are
\begin{equation}
\ell_{x_1}=\frac{\sigma_t^4}{\alpha_t^2}\ell_s,\qquad
\ell_{x_0}=\sigma_t^2\ell_s,\qquad
\ell_u=\frac{\sigma_t^2\Delta_t^2}{\alpha_t^2}\ell_s,\qquad
\ell_v=\frac{\sigma_t^2R_t^2}{\alpha_t^2}\ell_s.
\label{eq:p1-score-error-scalings}
\end{equation}
In particular, all five objectives become the same score objective
$\frac12\int_0^1\lambda(t)\E\ell_s\dd t$ if
\[
\omega_s=\lambda(t),\quad
\omega_{x_1}=\lambda(t)\frac{\alpha_t^2}{\sigma_t^4},\quad
\omega_{x_0}=\frac{\lambda(t)}{\sigma_t^2},\quad
\omega_u=\lambda(t)\frac{\alpha_t^2}{\sigma_t^2\Delta_t^2},\quad
\omega_v=\lambda(t)\frac{\alpha_t^2}{\sigma_t^2R_t^2}.
\]
Thus, unweighted alternative losses apply different time weights to score prediction. 

\clearpage
\part{A Forward Process from Data to Gaussian Noise and Its Reverse}\label{part:reverse}

\section{Setup}
\label{sec:noising-setup}

In this part, the forward noising clock \(t\in[0,1]\) runs from data
to noise. Draw \(X_0\sim p_0=p_{\mathrm{data}}\) and
\(X_1\sim p_1=\N(0,I)\) independently in \(\R^d\). Let \(\alpha_t\)
and \(\sigma_t\) be nonnegative schedules, continuous on \([0,1]\)
and differentiable on \((0,1)\), with
\[
\alpha_0=1,\qquad \sigma_0=0,\qquad
\alpha_1=0,\qquad \sigma_1=1.
\]
Typically, \(\alpha_t\) decreases while \(\sigma_t\) increases. Define
the interpolated random variable \(X_t\) and its law by
\[
X_t:=\alpha_t X_0+\sigma_t X_1,
\qquad
p_t:=\operatorname{Law}(X_t).
\]
The family \(\{p_t\}_{t\in[0,1]}\) is the \emph{probability path}.
It specifies the distribution at each forward time. Reverse-time
sampling follows this path from \(t=1\) toward \(t=0\).

For \(\sigma_t>0\), conditioning on the clean endpoint \(X_0=x_0\)
gives the Gaussian density
\begin{equation}
p_t(x\mid x_0):=\N(x;\alpha_t x_0,\sigma_t^2 I).
\label{eq:conditional-kernel}
\end{equation}
Averaging this conditional density over the data density \(p_0(x_0)\)
gives the marginal density
\begin{equation}
p_t(x)=\int_{\R^d}p_t(x\mid x_0)\,p_0(x_0)\dd x_0.
\label{eq:marginal-forward-density}
\end{equation}
At \(t=0\), the conditional law given \(X_0=x_0\) is a point mass at
\(x_0\). Wherever the marginal density is positive and differentiable,
its \emph{score} is \(s_t(x):=\grad\log p_t(x)\). All gradients,
divergences, and Laplacians act on the state variable \(x\).

For the forward SDE, assume \(\alpha_t>0\) on \((0,1)\) and define
\begin{equation}
f_t:=\frac{\dot{\alpha}_t}{\alpha_t},
\qquad
g_t^2:=\frac{\dd}{\dd t}\sigma_t^2
       -2\frac{\dot{\alpha}_t}{\alpha_t}\sigma_t^2.
\label{eq:def-f-g}
\end{equation}
Assume \(g_t^2>0\) on \((0,1)\), and take \(g_t\) to be its positive
square root. We write \(W_t\) for the standard Brownian motion that
drives the SDE in increasing forward time. If an endpoint is singular,
we work on a truncated interval \([a,b]\subset(0,1)\). The initial
law for that interval is \(p_a\) in forward time and \(p_b\) in
reverse time.

\begin{figure}[tb]
    \centering
    \includegraphics[width=\linewidth]{
        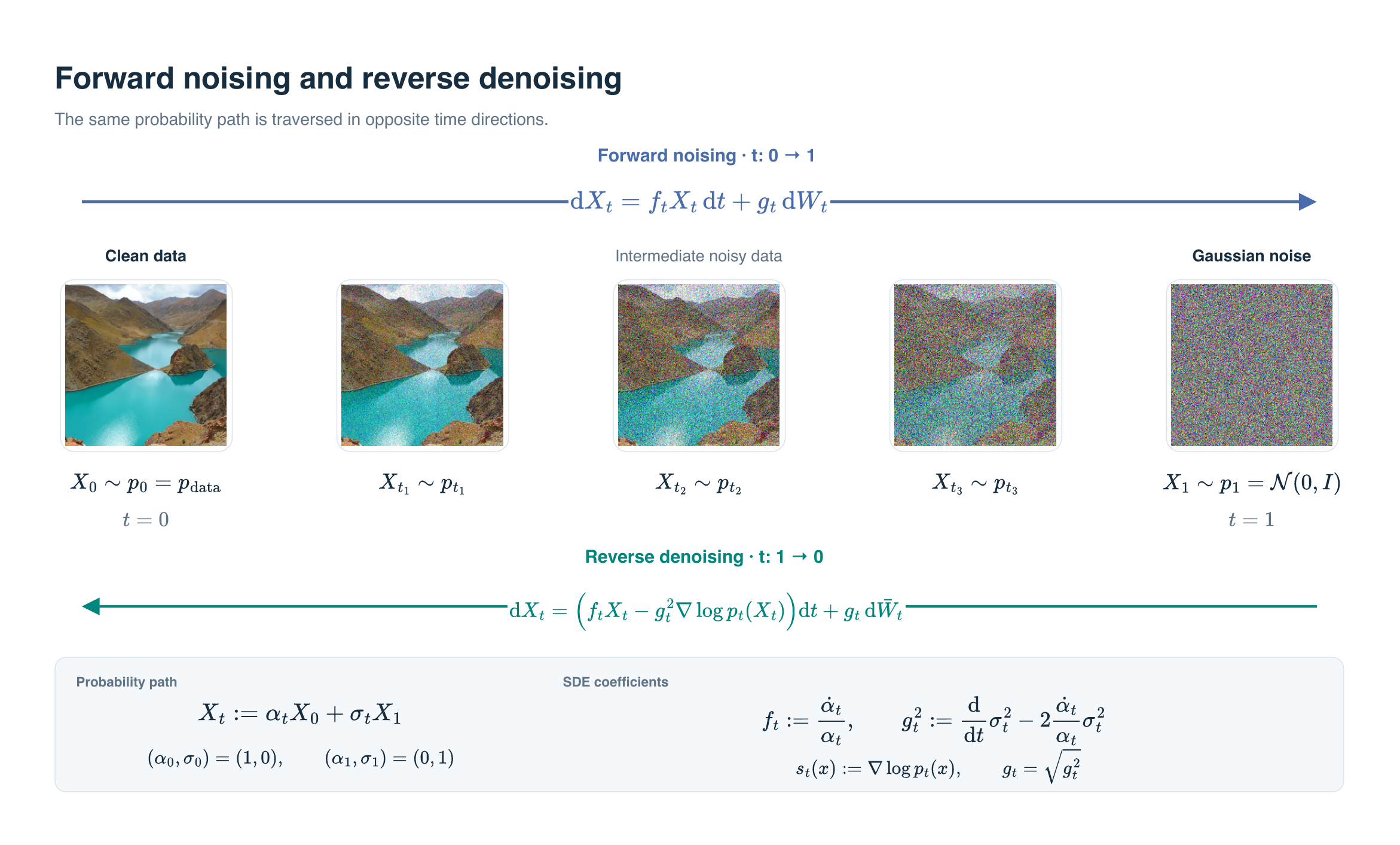
    }
    \caption{
        Forward noising and reverse denoising processes.
        The forward process gradually corrupts
        $X_0\sim p_0=p_{\mathrm{data}}$ into
        $X_1\sim p_1=\mathcal{N}(0,I)$ as $t$ increases from $0$ to $1$.
        The reverse process generates data from noise as $t$
        decreases from $1$ to $0$.
        Each arrow is labeled with its corresponding SDE.
    }
    \label{fig:noising-denoising}
\end{figure}
\section{Conditional Forward Process}
\label{sec:forward-conditional}

\subsection{Conditional Gaussian forward path}
Fix a clean data point $x_0$. The conditional noising path is
\begin{equation}
p_t(x\mid x_0)
:=
\N(x\mid \alpha_t x_0,\sigma_t^2I).
\label{eq:section7-conditional-noise-kernel}
\end{equation}
At $t=0$, it collapses to the clean data point $x_0$; at $t=1$, it has the Gaussian distribution $\mathcal{N}(0,I)$.

Its conditional score is
\begin{equation}
\grad\log p_t(x\mid x_0)
=
-\frac{x-\alpha_t x_0}{\sigma_t^2}.
\label{eq:section8-conditional-noise-score}
\end{equation}
\subsection{Conditional forward SDE}

\begin{theorem}[Conditional forward SDE]
\label{thm:conditional-forward-sde}
For every fixed $x_0$, the SDE
\begin{equation}
\dd X_t = f_t X_t \dd t + g_t \dd W_t,
\qquad
X_0=x_0,
\label{eq:conditional-forward-sde}
\end{equation}
has conditional density path \eqref{eq:conditional-kernel}. Equivalently, the density \eqref{eq:conditional-kernel} satisfies the conditional Fokker--Planck equation
\begin{equation}
\partial_t p_t(x \mid x_0)
= -\diver\bigl(f_t x\,p_t(x \mid x_0)\bigr)
+ \frac{1}{2}g_t^2 \lap p_t(x \mid x_0).
\label{eq:conditional-fp}
\end{equation}
\end{theorem}

\begin{proof}
We keep the notation $r_t(x)=x-\alpha_t x_0$. Differentiating the conditional Gaussian density gives:
\begin{equation}
\partial_t p_t(x \mid x_0)
= p_t(x \mid x_0)\left[
-d\frac{\dot{\sigma}_t}{\sigma_t}
+ \frac{\dot{\alpha}_t}{\sigma_t^2}r_t(x)^\top x_0
+ \frac{\dot{\sigma}_t}{\sigma_t^3}\|r_t(x)\|_2^2
\right].
\label{eq:conditional-fp-lhs}
\end{equation}

We now compute the right-hand side of \eqref{eq:conditional-fp}. First,
\[
\grad p_t(x \mid x_0)
= -p_t(x \mid x_0)\frac{r_t(x)}{\sigma_t^2}.
\]
Hence
\begin{align*}
\lap p_t(x \mid x_0)
&= \diver\!\left(-p_t(x \mid x_0)\frac{r_t(x)}{\sigma_t^2}\right)\\
&= -\frac{1}{\sigma_t^2}\left(r_t(x)^\top \grad p_t(x \mid x_0) + p_t(x \mid x_0)\diver r_t(x)\right)\\
&= -\frac{1}{\sigma_t^2}\left(-p_t(x \mid x_0)\frac{\|r_t(x)\|_2^2}{\sigma_t^2} + d\,p_t(x \mid x_0)\right)\\
&= p_t(x \mid x_0)\left(\frac{\|r_t(x)\|_2^2}{\sigma_t^4}-\frac{d}{\sigma_t^2}\right).
\end{align*}
Next,
\begin{align*}
\diver\bigl(f_t x\,p_t(x \mid x_0)\bigr)
&= f_t \diver\bigl(x\,p_t(x \mid x_0)\bigr)\\
&= f_t\left(d\,p_t(x \mid x_0)+x^\top \grad p_t(x \mid x_0)\right)\\
&= f_t d\,p_t(x \mid x_0)-f_t p_t(x \mid x_0)\frac{x^\top r_t(x)}{\sigma_t^2}.
\end{align*}
Therefore
\begin{align*}
&-\diver\bigl(f_t x\,p_t(x \mid x_0)\bigr)
+ \frac{1}{2}g_t^2 \lap p_t(x \mid x_0)\\
&\qquad=
p_t(x \mid x_0)\left[
-f_t d
+ f_t\frac{x^\top r_t(x)}{\sigma_t^2}
+ \frac{g_t^2}{2}\left(\frac{\|r_t(x)\|_2^2}{\sigma_t^4}-\frac{d}{\sigma_t^2}\right)
\right].
\end{align*}
Because $x=r_t(x)+\alpha_t x_0$,
\[
x^\top r_t(x)=\|r_t(x)\|_2^2+\alpha_t x_0^\top r_t(x).
\]
Also $f_t\alpha_t=\dot{\alpha}_t$, and
\[
\frac{g_t^2}{2\sigma_t^2}
= \frac{\dot{\sigma}_t}{\sigma_t}-f_t,
\qquad
\frac{g_t^2}{2\sigma_t^4}
= \frac{\dot{\sigma}_t}{\sigma_t^3}-\frac{f_t}{\sigma_t^2}.
\]
Substituting these identities yields
\begin{align*}
&-\diver\bigl(f_t x\,p_t(x \mid x_0)\bigr)
+ \frac{1}{2}g_t^2 \lap p_t(x \mid x_0)\\
&\qquad=
p_t(x \mid x_0)\left[
-d\frac{\dot{\sigma}_t}{\sigma_t}
+ \frac{\dot{\alpha}_t}{\sigma_t^2}r_t(x)^\top x_0
+ \frac{\dot{\sigma}_t}{\sigma_t^3}\|r_t(x)\|_2^2
\right].
\end{align*}
This matches \eqref{eq:conditional-fp-lhs}, so \eqref{eq:conditional-fp} holds exactly.
\end{proof}

\section{Marginalized Forward Process}
\label{sec:forward-marginal}

\subsection{Marginalized forward path}

The marginalized forward process is the unconditional family $(X_t)_{0\le t\le 1}$ with density path
\begin{equation}
\label{eq:marginal-path}
p_t(x)=\int p_t(x \mid x_0)p_0(x_0)\dd x_0.
\end{equation}

\subsection{Marginal forward SDE}

\begin{theorem}[Marginalized forward SDE]
The marginal forward process is generated by the SDE
\begin{equation}
\dd X_t = f_t X_t \dd t + g_t \dd W_t,
\qquad
X_0 \sim p_0.
\label{eq:marginal-forward-sde}
\end{equation}
Equivalently, the marginal density path satisfies the marginal Fokker--Planck equation
\begin{equation}
\partial_t p_t(x)
= -\diver\bigl(f_t x\,p_t(x)\bigr)
+ \frac{1}{2}g_t^2 \lap p_t(x).
\label{eq:marginal-forward-fp}
\end{equation}
\end{theorem}

\begin{proof}
For every fixed $x_0$, Theorem~\ref{thm:conditional-forward-sde} gives
\[
\partial_t p_t(x \mid x_0)
= -\diver\bigl(f_t x\,p_t(x \mid x_0)\bigr)
+ \frac{1}{2}g_t^2 \lap p_t(x \mid x_0).
\]
Integrating both sides against $p_0(x_0)\dd x_0$ gives
\begin{align*}
\partial_t p_t(x)
&= -\int \diver\bigl(f_t x\,p_t(x \mid x_0)\bigr)p_0(x_0)\dd x_0
+ \frac{1}{2}g_t^2\int \lap p_t(x \mid x_0)p_0(x_0)\dd x_0\\
&= -\diver\left(f_t x\int p_t(x \mid x_0)p_0(x_0)\dd x_0\right)
+ \frac{1}{2}g_t^2 \lap\left(\int p_t(x \mid x_0)p_0(x_0)\dd x_0\right)\\
&= -\diver\bigl(f_t x\,p_t(x)\bigr)
+ \frac{1}{2}g_t^2 \lap p_t(x).
\end{align*}
This is exactly \eqref{eq:marginal-forward-fp}, the Fokker--Planck equation of \eqref{eq:marginal-forward-sde}.
\end{proof}

\section{Reverse Process and Reverse Dynamics}
\label{sec:reverse-dynamics}

\subsection{Definition of the reverse process}

To describe reverse-time dynamics, we introduce the reverse-time variable
\begin{equation}
\tau := 1-t,
\label{eq:def-tau}
\end{equation}

\begin{definition}[Reverse process]
Given the forward marginal process $(X_t)_{0\le t\le 1}$, the reverse process is
\begin{equation}
Y_\tau := X_{1-\tau},
\qquad
0\le \tau \le 1.
\label{eq:reverse-process-definition}
\end{equation}
Its density is
\begin{equation}
q_\tau(x):=p_{1-\tau}(x).
\label{eq:reverse-process-density}
\end{equation}
\end{definition}

When the reverse SDE is written in the reverse-time variable $\tau$, its driving Brownian motion is denoted by $W^{\mathrm{rev}}_\tau$. When the same reverse SDE is written using the original label $t:1\to 0$, we denote the corresponding reverse-time Brownian motion by $\bar W_t$, where
\[
\bar W_t:=W^{\mathrm{rev}}_{1-t}.
\]
Thus $W^{\mathrm{rev}}_\tau$ and $\bar W_t$ represent the same reverse-time noise under two different time parametrizations. Appendix~\ref{app:reverse-brownian} gives a detailed discussion.

\subsection{Reverse density equation}

\begin{proposition}[PDE for the reverse density path]
The reverse density $q_\tau$ satisfies
\begin{equation}
\partial_\tau q_\tau(x)
= \diver\bigl(f_{1-\tau}x\,q_\tau(x)\bigr)
- \frac{1}{2}g_{1-\tau}^2 \lap q_\tau(x).
\label{eq:reverse-density-pde}
\end{equation}
\end{proposition}

\begin{proof}
Since $q_\tau(x)=p_{1-\tau}(x)$,
\[
\partial_\tau q_\tau(x)
= -\partial_t p_t(x)\big|_{t=1-\tau}.
\]
Using the forward marginal Fokker--Planck equation \eqref{eq:marginal-forward-fp},
\[
\partial_t p_t(x)
= -\diver\bigl(f_t x\,p_t(x)\bigr) + \frac{1}{2}g_t^2 \lap p_t(x),
\]
we obtain
\[
\partial_\tau q_\tau(x)
= \diver\bigl(f_{1-\tau}x\,q_\tau(x)\bigr)
- \frac{1}{2}g_{1-\tau}^2 \lap q_\tau(x).
\qedhere
\]
\end{proof}

\subsection{Reverse SDE}

The reverse-time diffusion viewpoint used below is classical in stochastic analysis \cite{anderson1982} and underlies the modern reverse-SDE formulation of score-based generative modeling \cite{song2021sde}.

\begin{theorem}[Reverse SDE in $\tau$]
Define
\begin{equation}
b_\tau^{\mathrm{rev}}(x)
:=
-f_{1-\tau}x + g_{1-\tau}^2 \grad\log q_\tau(x).
\label{eq:reverse-sde-drift-tau}
\end{equation}
Then the SDE
\begin{equation}
\dd Y_\tau
= b_\tau^{\mathrm{rev}}(Y_\tau)\dd \tau + g_{1-\tau}\dd W^{\mathrm{rev}}_\tau,
\qquad
Y_0 \sim q_0=p_1,
\label{eq:reverse-sde-tau}
\end{equation}
has density path $q_\tau$.
\end{theorem}

\begin{proof}
This is equivalent to proving that $q_\tau$ satisfies the Fokker--Planck equation
\[
\partial_\tau q_\tau(x)
= -\diver\bigl(b_\tau^{\mathrm{rev}}(x) q_\tau(x)\bigr)
+ \frac{1}{2}g_{1-\tau}^2 \lap q_\tau(x).
\]

Substitute \eqref{eq:reverse-sde-drift-tau}:
\begin{align*}
&-\diver\bigl(b_\tau^{\mathrm{rev}}(x)q_\tau(x)\bigr)
+ \frac{1}{2}g_{1-\tau}^2 \lap q_\tau(x)\\
&\qquad=
-\diver\left(\left[-f_{1-\tau}x + g_{1-\tau}^2 \grad\log q_\tau(x)\right]q_\tau(x)\right)
+ \frac{1}{2}g_{1-\tau}^2 \lap q_\tau(x)\\
&\qquad=
\diver\bigl(f_{1-\tau}x\,q_\tau(x)\bigr)
- g_{1-\tau}^2 \diver\bigl(q_\tau(x)\grad\log q_\tau(x)\bigr)
+ \frac{1}{2}g_{1-\tau}^2 \lap q_\tau(x).
\end{align*}
Since $q_\tau \grad\log q_\tau=\grad q_\tau$, we have
\[
\diver\bigl(q_\tau(x)\grad\log q_\tau(x)\bigr)=\lap q_\tau(x).
\]
Hence
\[
-\diver\bigl(b_\tau^{\mathrm{rev}}(x) q_\tau(x)\bigr)
+ \frac{1}{2}g_{1-\tau}^2 \lap q_\tau(x)
= \diver\bigl(f_{1-\tau}x\,q_\tau(x)\bigr)
- \frac{1}{2}g_{1-\tau}^2 \lap q_\tau(x),
\]
which is exactly \eqref{eq:reverse-density-pde}. Therefore $q_\tau$ satisfies the Fokker--Planck equation
\end{proof}

\begin{corollary}[Reverse SDE written in the original time label]
Rewriting \eqref{eq:reverse-sde-tau} in the original time label $t$ gives
\begin{equation}
\dd X_t
= \Bigl(f_t X_t - g_t^2 \grad\log p_t(X_t)\Bigr)\dd t
+ g_t \dd \bar W_t,
\qquad
t:1\to 0,\qquad X_1\sim p_1
\label{eq:reverse-sde-t}
\end{equation}
\end{corollary}

\begin{proof}
Set $\tau=1-t$ and $Y_\tau=X_t$. Then $q_\tau=p_t$, and the drift \eqref{eq:reverse-sde-drift-tau} becomes
\[
b_\tau^{\mathrm{rev}}(x)=-f_t x+g_t^2\grad\log p_t(x).
\]
Writing the same diffusion in backward $t$ notation gives \eqref{eq:reverse-sde-t}.
\end{proof}

\section{Learning the Score}
\label{sec:score-learning}
\subsection{The unknown marginal score}
Recall that the reverse SDE is
\begin{equation}
\dd X_t=\bigl(f_tX_t-g_t^2\grad\log p_t(X_t)\bigr)\dd t
+g_t\dd\bar W_t,\qquad t:1\to0.
\label{eq:section4-reverse-sde-recall}
\end{equation}
The schedules determine $f_t$ and $g_t$. The unknown field is
$s_t(x):=\grad\log p_t(x)$, so it is this marginal score that must be learned from data. The input $X_t$ is sampled from the conditional Gaussian path with clean endpoint $X_0$; its conditional score is known even though the marginal score is generally unavailable.

\subsection{Marginal score matching}
A natural strategy is to approximate the marginal score with a neural network,
\[
s_\theta(x,t)\approx s_t(x)=\grad\log p_t(x).
\]
If the marginal score $s_t(x)=\grad\log p_t(x)$ were directly available, one could train a score model $s_\theta(x,t)$ by
\begin{equation}
\mathcal L_{\mathrm{SM}}(\theta)
:=
\frac{1}{2}\int_0^1 \lambda(t)\,
\E_{X_t\sim p_t}\left[
\left\|s_\theta(X_t,t)-\grad\log p_t(X_t)\right\|_2^2
\right]\dd t.
\label{eq:raw-score-matching}
\end{equation}
Here $\lambda(t)\ge0$ is a time-weighting function, chosen so that the objective is finite. In practice, the marginal score is unavailable, so this objective needs a tractable conditional counterpart.

\subsection{Conditional score matching}
Draw $X_0\sim p_0$ and $X_t\sim p_t(\cdot\mid X_0)$. The Gaussian conditional path has the known score
\[
\grad_x\log p_t(x\mid x_0)
=-\frac{x-\alpha_t x_0}{\sigma_t^2}.
\]
Thus one can train with the conditional score itself as the target:
\begin{equation}
\mathcal L_{\mathrm{CSM}}(\theta)
:=
\frac{1}{2}\int_0^1 \lambda(t)\,
\E_{X_0,\;X_t\sim p_t(\cdot\mid X_0)}\left[
\|s_\theta(X_t,t)-\grad\log p_t(X_t\mid X_0)\|_2^2
\right]\dd t.
\label{eq:part2-conditional-score-matching}
\end{equation}
\begin{theorem}[Conditional score matching is a proxy for marginal score matching]
The objectives \eqref{eq:raw-score-matching} and \eqref{eq:part2-conditional-score-matching} differ only by a constant independent of $\theta$. Equivalently,
\[
\mathcal L_{\mathrm{CSM}}(\theta)
=
\mathcal L_{\mathrm{SM}}(\theta)+C.
\]
\end{theorem}

\begin{proof}
Fix $t$ and define
\[
U:=X_t,
\qquad
V:=\grad\log p_t(X_t\mid X_0),
\qquad
f(U):=s_\theta(X_t,t).
\]
By Fisher's identity,
\[
\E[V\mid U]
=
\grad\log p_t(U).
\]
Applying the orthogonality identity from Appendix~\ref{app:orthogonality} gives
\[
\E\|f(U)-V\|_2^2
=
\E\|f(U)-\E[V\mid U]\|_2^2
+ \E\|V-\E[V\mid U]\|_2^2.
\]
Hence
\[
\E\|s_\theta(X_t,t)-\grad\log p_t(X_t\mid X_0)\|_2^2
=
\E\|s_\theta(X_t,t)-\grad\log p_t(X_t)\|_2^2
+ C_t,
\]
where $C_t$ does not depend on $\theta$. Integrating over $t$ yields the result.
\end{proof}

\subsection{The learned sampling SDE}
Once the score has been learned, the reverse SDE is
\begin{equation}
\dd X_t=\bigl(f_tX_t-g_t^2s_\theta(X_t,t)\bigr)\dd t
+g_t\dd\bar W_t,\qquad t:1\to0,\quad X_1\sim p_1.
\label{eq:part2-learned-score-sde}
\end{equation}
With the exact marginal score and the prescribed terminal initialization, this SDE follows the data-to-noise marginal path backward. The learned score approximates this sampling dynamics directly.

\section{Prediction Targets and Loss Functions}
\label{sec:prediction-parameterizations}
The preceding section learns the marginal score directly. We now 
introduce other training objectives and explain their relations.

Here the clean variable is $X_0$, and the clock runs from data to noise:
\[
X_t=\alpha_tX_0+\sigma_tX_1,
\qquad X_0\sim p_0,\quad X_1\sim p_1=\N(0,I).
\]
All conversions are used at interior times with $\alpha_t,\sigma_t>0$; any additional denominator must be nonzero. Time weights are nonnegative; weights or interior cutoffs are chosen so that the objectives are finite.

\subsection{Alternative training objectives}
We introduce five training objectives here:
$$
X_0,\ X_1,\ 
S_t:=-\frac{X_1}{\sigma_t},\ 
U_t:=\dot\alpha_tX_0+\dot\sigma_tX_1,\ 
V_t:=\alpha_tX_1-\sigma_tX_0.
$$

The five conditional objectives are
\begin{align}
\mathcal L_{x_0}(\theta)
&:=\frac12\int_0^1\omega_{x_0}(t)\,
\E\| x_{0,\theta}(X_t,t)-X_0\|_2^2\dd t,
\label{eq:prediction-loss-x}\\
\mathcal L_{x_1}(\theta)
&:=\frac12\int_0^1\omega_{x_1}(t)\,
\E\|x_{1,\theta}(X_t,t)-X_1\|_2^2\dd t,
\label{eq:prediction-loss-e}\\
\mathcal L_s(\theta)
&:=\frac12\int_0^1\omega_s(t)\,
\E\|s_\theta(X_t,t)-S_t\|_2^2\dd t,
\label{eq:prediction-loss-s}\\
\mathcal L_u(\theta)
&:=\frac12\int_0^1\omega_u(t)\,
\E\|u_\theta(X_t,t)-U_t\|_2^2\dd t,
\label{eq:prediction-loss-u}\\
\mathcal L_v(\theta)
&:=\frac12\int_0^1\omega_v(t)\,
\E\|v_\theta(X_t,t)-V_t\|_2^2\dd t.
\label{eq:prediction-loss-v}
\end{align}
Expectations are over $X_0$ and independent Gaussian noise. 

\subsection{Connections to score prediction}

Write
\begin{equation}
\Delta_t:=\alpha_t\dot\sigma_t-\sigma_t\dot\alpha_t,
\qquad R_t:=\alpha_t^2+\sigma_t^2.
\label{eq:prediction-determinant}
\end{equation}
Every alternative predictor gives a score directly:
\begin{equation}
\begin{aligned}
s_\theta^{(x_0)}(X_t,t)
&:=\frac{\alpha_t x_{0,\theta}(X_t,t)-X_t}{\sigma_t^2},&
s_\theta^{(x_1)}(X_t,t)
&:=-\frac{x_{1,\theta}(X_t,t)}{\sigma_t},\\
s_\theta^{(u)}(X_t,t)
&:=\frac{\dot\alpha_tX_t-\alpha_tu_\theta(X_t,t)}{\sigma_t\Delta_t},&
s_\theta^{(v)}(X_t,t)
&:=-\frac{\sigma_tX_t+\alpha_tv_\theta(X_t,t)}{\sigma_tR_t}.
\end{aligned}
\label{eq:prediction-all-score-conversions}
\end{equation}
The $U$ conversion requires $\Delta_t\ne0$, and the $V$ conversion requires $R_t>0$.

\paragraph{Comparing the objectives with score prediction.}
Use the conditional score squared error as the common reference:
\[
\ell_s:=\|s_\theta(X_t,t)-S_t\|_2^2.
\]
Let $\ell_{x_0},\ell_{x_1},\ell_u,\ell_v$ denote the squared errors inside the other objectives.

\begin{theorem}[Relations to the score-prediction loss]
For mutually converted outputs, the following identities hold sample by sample:
\begin{equation}
\ell_{x_0}=\frac{\sigma_t^4}{\alpha_t^2}\ell_s,\qquad
\ell_{x_1}=\sigma_t^2\ell_s,\qquad
\ell_u=\frac{\sigma_t^2\Delta_t^2}{\alpha_t^2}\ell_s,\qquad
\ell_v=\frac{\sigma_t^2R_t^2}{\alpha_t^2}\ell_s.
\label{eq:prediction-error-scalings}
\end{equation}
\end{theorem}

\clearpage

\section{Connecting DDPM and the VP-SDE}
\label{sec:ddpm-vp-sde}

\subsection{Forward process}

Let $0=t_0<t_1<\cdots<t_N=1$ be the comparison grid, and let $p_0$
be the data distribution on $\mathbb R^d$. Uppercase state symbols denote
random variables, and lowercase state symbols denote their values. Throughout,
$I$ denotes the $d\times d$ identity matrix.

\paragraph{DDPM.}
The DDPM forward process is the discrete-time Markov chain
$\{\tilde{X}_k\}_{k=0}^N$, defined by \cite{ho2020ddpm}
\[
\tilde{X}_0\sim p_0,\qquad
\tilde{X}_k=\sqrt{a_k}\,\tilde{X}_{k-1}+\sqrt{b_k}\,Z_k,
\qquad k=1,\ldots,N,
\]
where $a_k=1-b_k$, $0<b_k<1$, and $Z_1,\ldots,Z_N$ are independent
$\N(0,I)$ random vectors, independent of $\tilde{X}_0$.
Its transition densities are
\begin{equation}
q(\tilde{x}_k\mid\tilde{x}_{k-1})
=\N(\tilde{x}_k;\sqrt{a_k}\,\tilde{x}_{k-1},b_k I).
\label{eq:ddpm-forward-step}
\end{equation}
Set $\bar a_0=1$ and $\bar a_k=\prod_{i=1}^k a_i$.
Given $\tilde{X}_0=\tilde{x}_0$, each step scales
the previous state by $\sqrt{a_k}$ and adds independent Gaussian noise.
The conditional mean is therefore $\sqrt{\bar a_k}\,\tilde{x}_0$. Starting from
zero conditional variance, the variance coefficient after step $k$ is
$a_k(1-\bar a_{k-1})+b_k=1-\bar a_k$ by induction. Thus
\begin{equation}
q(\tilde{x}_k\mid\tilde{x}_0)
=\N(\tilde{x}_k;\sqrt{\bar a_k}\,\tilde{x}_0,(1-\bar a_k)I).
\label{eq:ddpm-forward-marginal}
\end{equation}

\paragraph{VP-SDE.}
The VP-SDE forward process is the continuous-time process
$\{X_t\}_{t\in[0,1]}$ defined as the solution of the
variance-preserving stochastic differential equation \cite{song2021sde}
\begin{equation}
\dd X_t=-\tfrac12\beta(t)X_t\dd t+\sqrt{\beta(t)}\dd W_t,
\qquad X_0\sim p_0,\quad 0\le t\le 1,
\label{eq:ddpm-vp-forward-sde}
\end{equation}
where $\beta(t)$ is deterministic, positive, and integrable, and
$\{W_t\}_{t\in[0,1]}$ is a standard $d$-dimensional Brownian motion
independent of $X_0$. Set
\[
\alpha_t=\exp\!\left(-\tfrac12\int_0^t\beta(s)\dd s\right),
\qquad \sigma_t=\sqrt{1-\alpha_t^2}.
\]
Since $\dd(\alpha_t^{-1})=\tfrac12\beta(t)\alpha_t^{-1}\dd t$, the product
rule cancels the drift and gives
$\dd(\alpha_t^{-1}X_t)=\alpha_t^{-1}\sqrt{\beta(t)}\dd W_t$. Integrating
from $u$ to $t$ for $0\le u<t\le 1$ yields
\[
X_t=\frac{\alpha_t}{\alpha_u}X_u+
    \int_u^t\frac{\alpha_t}{\alpha_s}\sqrt{\beta(s)}\,\dd W_s .
\]
The integral is centered Gaussian with deterministic integrand and uses
only Brownian increments after $u$, so it is independent of $X_u$. Its
covariance follows from the It\^{o} isometry
(Appendix~\ref{app:ito-isometry}) and
$\tfrac{\mathrm d}{\mathrm ds}(\alpha_t/\alpha_s)^2
=\beta(s)(\alpha_t/\alpha_s)^2$. In particular,
\[
I\int_u^t(\alpha_t/\alpha_s)^2\beta(s)\dd s
=I\bigl[(\alpha_t/\alpha_s)^2\bigr]_{s=u}^{s=t}
=\bigl[1-(\alpha_t/\alpha_u)^2\bigr]I .
\]
Consequently,
\begin{equation}
X_t\mid X_u=y
\sim\N\!\left(\frac{\alpha_t}{\alpha_u}y,
        \bigl[1-(\alpha_t/\alpha_u)^2\bigr]I\right).
\label{eq:vp-forward-transition}
\end{equation}
Taking $u=0$ gives $p_t(x\mid x_0)=\N(x;\alpha_t x_0,\sigma_t^2 I)$.

\paragraph{Comparison.}
On the same grid interval $[t_{k-1},t_k]$, set
\begin{equation}
h_k=\int_{t_{k-1}}^{t_k}\beta(s)\dd s .
\label{eq:integrated-step-size}
\end{equation}
For the same input state $y$, the one-step conditional laws are
\[
\begin{aligned}
\tilde{X}_k\mid \tilde{X}_{k-1}=y
  &\sim\N\!\left(\sqrt{a_k}\,y,b_k I\right)
    &&\text{(DDPM)},\\
X_{t_k}\mid X_{t_{k-1}}=y
  &\sim\N\!\left(e^{-h_k/2}y,(1-e^{-h_k})I\right)
    &&\text{(VP-SDE)}.
\end{aligned}
\]
The VP-SDE row follows from \eqref{eq:vp-forward-transition} because
$\alpha_{t_k}/\alpha_{t_{k-1}}=e^{-h_k/2}$. Equating the mean
coefficients and covariances gives $a_k=e^{-h_k}$ and
$b_k=1-e^{-h_k}$. A given DDPM schedule is matched by setting
$\beta(t)=-\log(a_k)/(t_k-t_{k-1})$ on each interval. For the same
clean input $x_0$, multiply the step coefficients to obtain
\begin{equation}
\alpha_{t_k}^2
 =e^{-\sum_{i=1}^k h_i}
 =\prod_{i=1}^k a_i=\bar a_k,\qquad
\sigma_{t_k}^2=1-\bar a_k .
\label{eq:discrete-continuous-identification}
\end{equation}
The clean-conditioned distributions therefore have identical parameters
\[
\begin{aligned}
\tilde{X}_k\mid\tilde{X}_0=x_0
 &\sim\N\!\left(\sqrt{\bar a_k}\,x_0,(1-\bar a_k)I\right)
    &&\text{(DDPM)},\\
X_{t_k}\mid X_0=x_0
 &\sim\N\!\left(\alpha_{t_k}x_0,\sigma_{t_k}^2 I\right)
    &&\text{(VP-SDE)}.
\end{aligned}
\]
Since the two forward processes are Markov and have the same initial law
and grid transition laws,
\[
(\tilde{X}_0,\ldots,\tilde{X}_N)
\overset{\mathrm d}{=}(X_{t_0},\ldots,X_{t_N}).
\]
\textbf{Hence, the DDPM forward process agrees in law with the VP-SDE
forward process sampled on the comparison grid when
$\beta(t)=-\log(a_k)/(t_k-t_{k-1})$ on each interval.}

\subsection{Training loss}

\paragraph{DDPM.}
Draw $k$ uniformly from $\{1,\ldots,N\}$, $\tilde{X}_0\sim p_0$, and
$\varepsilon\sim\N(0,I)$ independently. Using
\eqref{eq:ddpm-forward-marginal}, the simplified DDPM objective trains
$\epsilon_\theta$ to predict the Gaussian noise \cite{ho2020ddpm}
\begin{equation}
\begin{aligned}
\tilde{X}_k&=\sqrt{\bar a_k}\,\tilde{X}_0+\sqrt{1-\bar a_k}\,\varepsilon,\\
\mathcal L_{\mathrm{DDPM}}(\theta)
&=\E\bigl[\|\epsilon_\theta(\tilde{X}_k,k)-\varepsilon\|_2^2\bigr].
\end{aligned}
\label{eq:ddpm-simple-loss}
\end{equation}

\paragraph{VP-SDE.}
Denoising score matching trains $s_\theta$ using the conditional
score as a target
\cite{song2021sde}. To draw a training sample at a fixed $t\in(0,1)$, form
$X_t=\alpha_t X_0+\sigma_t\varepsilon$ from independent
$X_0\sim p_0$ and $\varepsilon\sim\N(0,I)$ using the conditional
marginal law. Differentiating the
conditional Gaussian log density gives
\[
\left.\grad_x\log p_t(x\mid X_0)\right|_{x=X_t}
=-\frac{X_t-\alpha_t X_0}{\sigma_t^2}
=-\frac{\varepsilon}{\sigma_t}.
\]
With time sampled uniformly on $(0,1)$ and positive weight $\lambda(t)$,
the objective is
\begin{equation}
\mathcal L_{\mathrm{VP}}(\theta;\lambda)
=\int_0^1\lambda(t)\,
  \E\left[\left\|s_\theta(X_t,t)+\frac{\varepsilon}{\sigma_t}
      \right\|_2^2\right]\dd t .
\label{eq:vp-training-loss}
\end{equation}

\paragraph{Comparison.}
For each fixed $k$, couple the initial draws as $\tilde{X}_0=X_0\sim p_0$
and use the same $\varepsilon$ to construct the two training samples
from their conditional marginal laws.
Equation \eqref{eq:discrete-continuous-identification} gives
$\tilde{X}_k=X_{t_k}$ for these coupled samples, so both models receive
the same input $x$. Express the same prediction in score coordinates
on the grid by setting
\begin{equation}
s_\theta(x,t_k)=-\frac{\epsilon_\theta(x,k)}{\sigma_{t_k}} .
\label{eq:ddpm-noise-to-score}
\end{equation}
Then the squared errors satisfy
\[
\|\epsilon_\theta(x,k)-\varepsilon\|_2^2
 =\left\|-\sigma_{t_k}s_\theta(x,t_k)-\varepsilon\right\|_2^2=\sigma_{t_k}^2
   \left\|s_\theta(x,t_k)+\frac{\varepsilon}{\sigma_{t_k}}\right\|_2^2 .
\]
Taking expectations over the common draws and averaging over $K$ gives
\begin{equation}
\mathcal L_{\mathrm{DDPM}}(\theta)
=\frac{1}{N}\sum_{k=1}^N \sigma_{t_k}^2\,
  \E\left[\left\|s_\theta(X_{t_k},t_k)
               +\frac{\varepsilon}{\sigma_{t_k}}\right\|_2^2\right] .
\label{eq:ddpm-discrete-score-loss}
\end{equation}
Thus the losses are exactly equal when the VP-SDE objective uses the
same grid sampling $t_K$ and weight $\lambda(t_k)=\sigma_{t_k}^2$.
\textbf{Hence, the DDPM training loss is equivalent to the VP-SDE training loss.}

\subsection{Reverse sampling}

In this subsection, we use $\{\tilde{X}_k\}_{k=0}^N$ and
$\{X_t\}_{t\in[0,1]}$ for the corresponding learned reverse samplers,
with the initializations specified below.

\paragraph{DDPM.}
The ancestral sampler starts at $\tilde{X}_N\sim\N(0,I)$ and, for
$k=N,\ldots,1$, updates the current state $\tilde{X}_k=x$ using the
fixed posterior variance choice \cite{ho2020ddpm}
\begin{equation}
\begin{aligned}
\tilde{X}_{k-1}
&=\frac{x-\dfrac{b_k}{\sqrt{1-\bar a_k}}\epsilon_\theta(x,k)}
        {\sqrt{a_k}}+\sqrt{\widetilde b_k}\,Z_k,\\
\widetilde b_k&=b_k\frac{1-\bar a_{k-1}}{1-\bar a_k},
\qquad Z_k\sim\N(0,I),
\end{aligned}
\label{eq:ddpm-ancestral-step}
\end{equation}
with fresh independent noise at each step.

\paragraph{VP-SDE.}
Start at $X_1\sim\N(0,I)$ and solve the reverse-time dynamics
\eqref{eq:reverse-sde-t} with the learned predictor $s_\theta$
\cite{song2021sde}
\begin{equation}
\dd X_t=-\beta(t)\bigl[\tfrac12 X_t+s_\theta(X_t,t)\bigr]\dd t
        +\sqrt{\beta(t)}\dd\bar W_t,\qquad t:1\to 0,
\label{eq:ddpm-vp-reverse-sde}
\end{equation}
where $\bar W_t$ is reverse-time Brownian motion. In a numerical step
from $t_k$ to $t_{k-1}$, freeze the current state $X_{t_k}=x$ and score predictor at
$(x,t_k)$. The drift integrates with
$\int_{t_k}^{t_{k-1}}-\beta(t)\dd t=h_k$. The Gaussian noise has
covariance $\int_{t_{k-1}}^{t_k}\beta(t)\dd t\,I=h_k I$ by the
It\^{o} isometry. The new state is
\begin{equation}
x+h_k\bigl[\tfrac12 x+s_\theta(x,t_k)\bigr]+\sqrt{h_k}\,Z_k,
\qquad Z_k\sim\N(0,I),
\label{eq:vp-reverse-numerical-step}
\end{equation}
with a fresh independent $Z_k$ at each step.

\paragraph{Comparison.}
Compare one step at the same $x$ and $t_k$, with the same Gaussian draw
$Z_k$ and related predictors as in \eqref{eq:ddpm-noise-to-score}.
Since $\sqrt{1-\bar a_k}=\sigma_{t_k}$, the DDPM mean takes the form
$(x+b_k s_\theta(x,t_k))/\sqrt{a_k}$. Substituting $a_k=e^{-h_k}$ and
$b_k=1-e^{-h_k}$ gives both updates in terms of $h_k$
\[
\begin{aligned}
\text{DDPM}\quad
 &e^{h_k/2}x+(e^{h_k/2}-e^{-h_k/2})s_\theta(x,t_k)
           +\sqrt{\widetilde b_k}\,Z_k,\\
\text{VP-SDE}\quad
 &(1+\tfrac12h_k)x+h_k s_\theta(x,t_k)+\sqrt{h_k}\,Z_k .
\end{aligned}
\]
Both noise terms have mean zero. Taylor expansion gives
$e^{\pm h_k/2}=1\pm h_k/2+O(h_k^2)$, and subtracting the two
expansions gives $e^{h_k/2}-e^{-h_k/2}=h_k+O(h_k^2)$. Thus the
conditional means differ by $O(h_k^2)$ at fixed $x$ and predictor
value. For the covariance, use
$\alpha_{t_k}^2=a_k\alpha_{t_{k-1}}^2$ to get
\[
\sigma_{t_k}^2-\sigma_{t_{k-1}}^2
 =\alpha_{t_{k-1}}^2-\alpha_{t_k}^2
 =b_k\alpha_{t_{k-1}}^2 .
\]
The DDPM variance coefficient is therefore
\[
\widetilde b_k
 =b_k\frac{\sigma_{t_{k-1}}^2}{\sigma_{t_k}^2}
 =b_k-\frac{b_k^2\alpha_{t_{k-1}}^2}{\sigma_{t_k}^2}.
\]
Since $b_k=1-e^{-h_k}=h_k+O(h_k^2)$, the correction is $O(h_k^2)$
when $t_k$ stays bounded away from zero, where $\sigma_{t_k}^2$
has a positive lower bound. 
\textbf{Hence, DDPM reverse sampling and VP-SDE reverse sampling agree to first order in the mean and variance.}

\clearpage
\part{SDE Solvers and Guided Diffusion}\label{part:solvers}

\section{Numerical Solvers}\label{sec:sampling}

\subsection{The forward generative SDE of Part~I}

We use the learned forward generative SDE
\eqref{eq:section8-learned-generative-sde-score} of
Part~\ref{part:generative}, with the same schedules $f_t,g_t$ and score
network $s_\theta(x,t)$.
\begin{equation}
\dd X_t=\bigl[f_tX_t+g_t^2s_\theta(X_t,t)\bigr]\dd t
       +g_t\dd W_t,
\qquad X_0\sim p_0=\N(0,I),\qquad t:0\to1.
\label{eq:part3-forward-generative-sde}
\end{equation}
Sampling proceeds in increasing time from noise to data. Choose a grid
$0\le t_0<\cdots<t_N\le1$ with steps $h_n:=t_{n+1}-t_n$.
In the numerical updates, $X_{t_n}$ denotes the computed state at time
$t_n$. As in Part~I, if a coefficient is singular at an endpoint, use
an interior grid and initialize from $p_{t_0}$ for that truncated path.

\subsection{Euler--Maruyama}

Over one step, the SDE integrates to a drift integral and a stochastic
integral. Euler--Maruyama evaluates both coefficients at the left
endpoint of the step \cite{higham2001simulation}.
\begin{equation}
X_{t_{n+1}}=X_{t_n}
 +h_n\bigl[f_{t_n}X_{t_n}+g_{t_n}^2s_\theta(X_{t_n},t_n)\bigr]
 +g_{t_n}\sqrt{h_n}Z_n,
\label{eq:direct-sde-euler}
\end{equation}
where $Z_n\sim\N(0,I)$ are independent across steps and independent of
the initial state. The factor $\sqrt{h_n}$ comes from the Brownian
increment $W_{t_{n+1}}-W_{t_n}\sim\N(0,h_nI)$.
\subsection{Stochastic Heun}

Stochastic Heun uses a second score evaluation to account for the
change in drift over the step. First, predict the endpoint as in \eqref{eq:direct-sde-euler}.
\[
X_{t_{n+1}}^{\mathrm{pred}}=X_{t_n}
 +h_n\bigl[f_{t_n}X_{t_n}+g_{t_n}^2s_\theta(X_{t_n},t_n)\bigr]
 +g_{t_n}\sqrt{h_n}Z_n.
\]
Then average the drift at the starting state and the predicted endpoint.
\begin{equation}
\begin{aligned}
X_{t_{n+1}}=X_{t_n}+\frac{h_n}{2}\bigl[\,&f_{t_n}X_{t_n}
 +g_{t_n}^2s_\theta(X_{t_n},t_n)\\
 &+f_{t_{n+1}}X_{t_{n+1}}^{\mathrm{pred}}
 +g_{t_{n+1}}^2s_\theta(X_{t_{n+1}}^{\mathrm{pred}},t_{n+1})\bigr]
 +g_{t_n}\sqrt{h_n}Z_n.
\end{aligned}
\label{eq:part3-stochastic-heun}
\end{equation}

\subsection{DPM-Solver for the probability-flow ODE}
\label{sec:part3-dpm-solver}

DPM-Solver integrates the probability-flow ODE associated with the
forward generative SDE \cite{lu2022dpmsolver}. For the exact score,
$\lap p_t=\diver(p_t\grad\log p_t)$ lets us combine the diffusion term
and the score term in the Fokker--Planck equation. The resulting
continuity equation has drift $f_tx+(g_t^2/2)\grad\log p_t(x)$.
Substituting the Part~I score network gives the deterministic sampler
\begin{equation}
\frac{\dd X_t}{\dd t}=f_tX_t+\frac{g_t^2}{2}s_\theta(X_t,t),
\qquad t:0\to1.
\label{eq:part3-probability-flow-ode}
\end{equation}
With the exact score and matching initialization, the ODE and SDE have
the same one-time marginals. With a learned score, each equation defines
its own approximate sampler.

On an interior interval with $\alpha_t,\sigma_t>0$, let
$\Lambda_t:=\log(\alpha_t/\sigma_t)$ be the half log signal-to-noise
ratio. By \eqref{eq:section8-generative-fg},
\[
\dot\Lambda_t=f_t-\frac{\dot\sigma_t}{\sigma_t}
              =\frac{g_t^2}{2\sigma_t^2}>0.
\]
Let $\tau(\ell)$ denote the inverse of $t\mapsto\Lambda_t$. Since
$f_t=\dot\alpha_t/\alpha_t$, the ODE gives
\[
\frac{\dd}{\dd t}\left(\frac{X_t}{\alpha_t}\right)
=\frac{g_t^2}{2\alpha_t}s_\theta(X_t,t).
\]
For $a<t$ in this interval, changing variables to $\Lambda$ yields the
exact ODE solution in integral form.
\begin{equation}
X_t=\frac{\alpha_t}{\alpha_a}X_a
+\alpha_t\int_{\Lambda_a}^{\Lambda_t}
 e^{-\ell}\sigma_{\tau(\ell)}
 s_\theta(X_{\tau(\ell)},\tau(\ell))\dd\ell.
\label{eq:part3-dpm-integral}
\end{equation}
DPM-Solver integrates the linear term and the exponential weight
analytically and approximates the scaled score inside the integral.

\paragraph{DPM-Solver-1.}
Let $\Delta\Lambda_n=\Lambda_{t_{n+1}}-\Lambda_{t_n}$. Holding the
scaled score fixed at the left endpoint gives the update
\begin{equation}
X_{t_{n+1}}=\frac{\alpha_{t_{n+1}}}{\alpha_{t_n}}X_{t_n}
 +\sigma_{t_{n+1}}\bigl(e^{\Delta\Lambda_n}-1\bigr)
  \sigma_{t_n}s_\theta(X_{t_n},t_n).
\label{eq:part3-dpm-solver-1}
\end{equation}

\paragraph{DPM-Solver-2.}
Choose the intermediate time $\tau_n$ halfway in $\Lambda$ so that
$\Lambda_{\tau_n}=\Lambda_{t_n}+\Delta\Lambda_n/2$. Take a first-order
step to $\tau_n$ and use the scaled score there for the full step.
\begin{equation}
\begin{aligned}
X_{\tau_n}^{\mathrm{pred}}
 &=\frac{\alpha_{\tau_n}}{\alpha_{t_n}}X_{t_n}
 +\sigma_{\tau_n}\bigl(e^{\Delta\Lambda_n/2}-1\bigr)
  \sigma_{t_n}s_\theta(X_{t_n},t_n),\\
X_{t_{n+1}}
 &=\frac{\alpha_{t_{n+1}}}{\alpha_{t_n}}X_{t_n}
 +\sigma_{t_{n+1}}\bigl(e^{\Delta\Lambda_n}-1\bigr)
  \sigma_{\tau_n}s_\theta(X_{\tau_n}^{\mathrm{pred}},\tau_n).
\end{aligned}
\label{eq:part3-dpm-solver-2}
\end{equation}
These methods use one and two score evaluations per step, respectively,
and attain first- and second-order for the ODE under the corresponding
regularity assumptions \cite{lu2022dpmsolver}. The paper also derives a
third-order method using two intermediate stages. In Part~I,
$\Lambda_t$ diverges at both ideal endpoints. These updates are written
for finite $\Lambda_{t_n}$ on an interior grid with the initialization
specified above.

\newpage
\section{Guided Diffusion}\label{sec:guidance}

We continue with the forward generative SDE of Part~\ref{part:generative},
which takes Gaussian noise at $t=0$ to data at $t=1$. To generate data
under a condition $c$, such as a class label or a text prompt, we use a
\emph{guided score}. The exact guided score is
$\grad_x\log p_t(x\mid c)$. Classifier guidance and classifier-free
guidance construct learned guided scores in different ways
\cite{dhariwal2021guided,ho2022cfg}.
Throughout this section, draw $(X_1,c)$ jointly from clean data and
conditions and draw $X_0\sim\N(0,I)$ independently. Use the interpolation
of Section~\ref{sec:part1-setup}; $p_t(x\mid c)$ denotes its density
given condition $c$. We use the same guidance scale $\gamma\ge0$ for both methods.

\subsection{Classifier guidance}

Classifier guidance combines the unconditional score network
$s_\theta(x,t)$ from Part~\ref{part:generative} with a classifier
$p_\phi(c\mid x,t)$ that takes both the state $x$ and the diffusion
time $t$ as inputs. Train it on the interpolated samples $X_t$ from
Section~\ref{sec:part1-setup}, using cross-entropy averaged uniformly
over $0<t<1$.
\begin{equation}
\mathcal L_{\mathrm{clf}}(\phi)
:=\int_0^1
\E_{X_0,X_1,c}\bigl[-\log p_\phi(c\mid X_t,t)\bigr]\dd t.
\label{eq:classifier-guidance-loss}
\end{equation}

By Bayes' rule, the exact guided score satisfies
\begin{equation}
\grad_x\log p_t(x\mid c)
=\grad_x\log p_t(x)+\grad_x\log p_t(c\mid x).
\label{eq:classifier-guidance-score}
\end{equation}
Here $p_t(c\mid x)$ is the posterior at diffusion time $t$. Using
$p_\phi(c\mid x,t)$ to estimate it gives the classifier-guided score
\[
s_\theta^{\mathrm{cg}}(x,t,c;\gamma)
:=s_\theta(x,t)+\gamma\grad_x\log p_\phi(c\mid x,t).
\]
Substituting it into the learned SDE of Part~\ref{part:generative}
gives
\begin{equation}
\dd X_t
=\bigl[f_tX_t+g_t^2s_\theta^{\mathrm{cg}}(X_t,t,c;\gamma)\bigr]\dd t
       +g_t\dd W_t,
\qquad t:0\to1.
\label{eq:classifier-guided-sde}
\end{equation}

\subsection{Classifier-free guidance}

Classifier-free guidance trains a single score network $s_\theta(x,t,c)$
with random condition dropout \cite{ho2022cfg}. At training time, replace
$c$ with a distinguished null condition $\varnothing$ with probability
$P_{\mathrm{drop}}\in(0,1)$, independently of $(X_0,X_1,c)$ and with
the same probability at every time. Write $\widetilde c$ for the
resulting condition: $\widetilde c=c$ when the condition is retained
and $\widetilde c=\varnothing$ when it is dropped. Use the Gaussian
score target $S_t$ from Part~\ref{part:generative}:
\[
S_t=\left.\grad_x\log p_t(x\mid X_1)\right|_{x=X_t}
   =-\frac{X_0}{\sigma_t},
\qquad 0<t<1,\quad \sigma_t>0.
\]
With a nonnegative time weight $\lambda(t)$ chosen so the loss is
finite, the mixed conditional score-matching objective is
\begin{equation}
\mathcal L_{\mathrm{cfg}}(\theta)
:=\frac12\int_0^1\lambda(t)\,
\E_{X_0,X_1,c,\widetilde c}
\left[\bigl\|s_\theta(X_t,t,\widetilde c)-S_t\bigr\|_2^2\right]\dd t.
\label{eq:cfg-loss}
\end{equation}

For a retained condition $c$, Fisher's identity gives the exact guided
score as the population regression target. For the null condition,
independence of the dropout decision gives the unconditional score.
Thus, with sufficient model capacity and successful training,
\[
s_\theta(x,t,c)\approx\grad_x\log p_t(x\mid c),
\qquad
s_\theta(x,t,\varnothing)\approx\grad_x\log p_t(x).
\]

At sampling time, combine the two outputs with the same scale
$\gamma\ge0$ used for classifier guidance:
\begin{equation}
s_\theta^{\mathrm{cfg}}(x,t,c;\gamma)
:=s_\theta(x,t,\varnothing)
  +\gamma\bigl[s_\theta(x,t,c)-s_\theta(x,t,\varnothing)\bigr].
\label{eq:cfg-predictor}
\end{equation}
The classifier-free guided forward SDE is
\begin{equation}
\dd X_t
=\bigl[f_tX_t+g_t^2s_\theta^{\mathrm{cfg}}(X_t,t,c;\gamma)\bigr]\dd t
 +g_t\dd W_t,
\qquad t:0\to1.
\label{eq:cfg-sde}
\end{equation}
At $\gamma=0$, the sampler uses the unconditional score. At $\gamma=1$,
it uses the guided score $s_\theta(x,t,c)$ and recovers the conditional
path when that score is exact and the initialization matches. Taking
$\gamma>1$ amplifies the difference between the guided and unconditional
scores and can improve adherence to $c$ at the cost of diversity.

\subsection{Using guidance with the numerical solvers}

Write $s_\theta^{\mathrm{guided}}(x,t,c;\gamma)$ for either guided score.
The Euler--Maruyama step for the guided SDE is
\[
X_{t_{n+1}}=X_{t_n}
 +h_n\bigl[f_{t_n}X_{t_n}+g_{t_n}^2
         s_\theta^{\mathrm{guided}}(X_{t_n},t_n,c;\gamma)\bigr]
 +g_{t_n}\sqrt{h_n}\,Z_n,
\qquad Z_n\sim\N(0,I).
\]
Stochastic Heun evaluates the same guided score at both stages, using
the same Brownian increment. For deterministic sampling, substitute
the guided score into \eqref{eq:part3-probability-flow-ode}:
\[
\frac{\dd X_t}{\dd t}
=f_tX_t+\frac{g_t^2}{2}
s_\theta^{\mathrm{guided}}(X_t,t,c;\gamma).
\]
DPM-Solver then uses this guided score at each stage.

\clearpage
\appendix

\section{Differential Operators}
\label{app:differential-operators}
This section introduces important differential operators. Let 
$x=(x_1,\dots,x_d)\in\R^d$.
\begin{definition}[Gradient]
If $f:\R^d\to\R$ is a scalar-valued differentiable function, its gradient is the vector field
\[
\grad f(x)
:=
\left(
\frac{\partial f}{\partial x_1}(x),
\dots,
\frac{\partial f}{\partial x_d}(x)
\right)^\top.
\]
\end{definition}

Thus the gradient maps a scalar function to a vector field.

\begin{definition}[Divergence]
If $v:\R^d\to\R^d$ is a differentiable vector field with components
\[
v(x)=\bigl(v_1(x),\dots,v_d(x)\bigr)^\top,
\]
its divergence is the scalar function
\[
\diver v(x)
:=
\sum_{i=1}^d \frac{\partial v_i}{\partial x_i}(x).
\]
\end{definition}

Thus the divergence maps a vector field to a scalar function.

\begin{definition}[Laplacian]
If $f:\R^d\to\R$ is twice differentiable, its Laplacian is
\[
\lap f(x)
:=
\diver(\grad f(x))
=
\sum_{i=1}^d \frac{\partial^2 f}{\partial x_i^2}(x).
\]
\end{definition}

Thus the Laplacian maps a scalar function to a scalar function.

For completeness, if $v:\R^d\to\R^d$ is a differentiable vector field, then $\grad v(x)$ denotes its Jacobian matrix,
\[
\grad v(x)
:=
\left(\frac{\partial v_i}{\partial x_j}(x)\right)_{i,j=1}^d.
\]
This notation should not be confused with the divergence $\diver v(x)$, which is the trace of this Jacobian matrix.

\section{Brownian Motion in Forward and Reverse Time}
\label{app:reverse-brownian}

This section distinguishes the forward diffusion $X_t$, its reversed
trajectory $Y_\tau$, and their Brownian drivers $W_t$,
$W^{\mathrm{rev}}_\tau$, and $\bar W_t$.

\subsection*{Standard Brownian motion and filtrations}

A filtration $(\mathcal H_u)_{0\le u\le1}$ is an increasing family of
$\sigma$-algebras representing the information available by time $u$.
An $\mathbb R^d$-valued process $B_u$ is a standard Brownian motion
with respect to $(\mathcal H_u)$ if:
\begin{enumerate}
    \item $B_u$ is adapted to $(\mathcal H_u)$, meaning that $B_u$ is
    $\mathcal H_u$-measurable;
    \item $B_0=0$ almost surely;
    \item its sample paths are continuous almost surely;
    \item for every $0\le v<u\le1$,
    \[
    B_u-B_v\sim\mathcal N(0,(u-v)I_d),
    \]
    and this increment is independent of $\mathcal H_v$.
\end{enumerate}
Thus the Brownian property depends on both the time direction and the
filtration.

\subsection*{Forward diffusion}

Let $X_t$ denote the forward diffusion, with density $p_t$, satisfying
\[
\dd X_t=b_t(X_t)\,\dd t+g_t\,\dd W_t,
\qquad t:0\to1.
\]
Here $W_t$ is the forward Brownian driver. The natural filtrations of
the state and its driver are
\[
\mathcal F_t^X:=\sigma(X_s:0\le s\le t),
\qquad
\mathcal F_t^W:=\sigma(W_s:0\le s\le t).
\]
The SDE is understood with respect to a forward filtration
$(\mathcal F_t)$ containing both natural filtrations.

\subsection*{Reverse diffusion}

The reversed trajectory is
\[
Y_\tau:=X_{1-\tau},\qquad \tau:0\to1.
\]
Its natural filtration is
\[
\mathcal G_\tau^Y
:=\sigma(Y_\rho:0\le\rho\le\tau)
=\sigma(X_s:1-\tau\le s\le1).
\]
As $\tau$ increases, this filtration reveals progressively more of the
original trajectory, starting from $X_1$.

Under the usual time-reversal assumptions, $Y_\tau$ satisfies
\[
\dd Y_\tau=
\left[-b_{1-\tau}(Y_\tau)
+g_{1-\tau}^2\nabla\log p_{1-\tau}(Y_\tau)\right]\dd\tau
+g_{1-\tau}\,\dd W^{\mathrm{rev}}_\tau.
\]
Here $W^{\mathrm{rev}}_\tau$ is the reverse Brownian driver, whose
natural filtration is
\[
\mathcal G_\tau^{W^{\mathrm{rev}}}
:=\sigma(W^{\mathrm{rev}}_\rho:0\le\rho\le\tau).
\]
We use a reverse filtration $(\mathcal G_\tau)$ containing both
$\mathcal G_\tau^Y$ and
$\mathcal G_\tau^{W^{\mathrm{rev}}}$.

Both $W_t$ and $W^{\mathrm{rev}}_\tau$ satisfy the definition above,
but with different clocks and filtrations. In general,
$W^{\mathrm{rev}}_\tau$ is not obtained by simply reversing $W_t$.

\subsection*{The backward time label}

To express the same reverse dynamics using the original label $t$,
set
\[
t=1-\tau,
\qquad
\bar W_t:=W^{\mathrm{rev}}_{1-t},
\qquad t:1\to0.
\]
The reverse SDE then becomes
\[
\dd X_t=
\left[b_t(X_t)-g_t^2\nabla\log p_t(X_t)\right]\dd t
+g_t\,\dd\bar W_t,
\qquad t:1\to0.
\]
The corresponding backward filtrations are
\begin{align*}
\bar{\mathcal G}_t^X
&:=\mathcal G_{1-t}^Y
=\sigma(X_s:t\le s\le1),\\
\bar{\mathcal G}_t^{\bar W}
&:=\mathcal G_{1-t}^{W^{\mathrm{rev}}}
=\sigma(\bar W_s:t\le s\le1),
\end{align*}
with working filtration
$\bar{\mathcal G}_t:=\mathcal G_{1-t}$.
These families increase as $t$ decreases. In particular,
$\bar W_1=0$, and for $0\le s<t\le1$,
\[
\bar W_s-\bar W_t\sim\mathcal N(0,(t-s)I_d),
\qquad
\bar W_s-\bar W_t
\text{ is independent of }\bar{\mathcal G}_t.
\]
Therefore $\bar W_t$ is the same Brownian driver as
$W^{\mathrm{rev}}_\tau$, expressed in decreasing $t$; its Brownian
interpretation follows the direction $t:1\to0$.

\section{It\^o Integrals and It\^o Isometry}
\label{app:ito-isometry}

This section justifies the stochastic-integral calculation in
Section~\ref{sec:ddpm-vp-sde}. Let $W$ be a standard Brownian motion
and let $\phi\in L^2([a,b])$ be deterministic. We show that
\[
\eta:=\int_a^b \phi(\tau)\,\dd W_\tau
\sim \N(0,V),
\qquad
V:=\int_a^b \phi(\tau)^2\,\dd\tau.
\]
Here $\E_W$ denotes expectation over the Brownian motion.
For deterministic functions, write
\[
\|f\|_2^2:=\int_a^b |f(\tau)|^2\,\dd\tau.
\]

\subsection*{Step 1: step-function integrands}

For a deterministic step function
\[
\psi(\tau)
=\sum_{j=0}^{m-1}c_j\mathbf{1}_{(u_j,u_{j+1}]}(\tau),
\qquad
a=u_0<\cdots<u_m=b,
\]
define
\[
I(\psi)
:=\sum_{j=0}^{m-1}c_j\Delta W_j,
\qquad
\Delta W_j:=W_{u_{j+1}}-W_{u_j}.
\]
The increments are independent, with
$\Delta W_j\sim\N(0,u_{j+1}-u_j)$. In particular,
\[
\E_W[\Delta W_j]=0,
\qquad
\E_W[\Delta W_i\Delta W_j]
=
\begin{cases}
u_{j+1}-u_j,& i=j,\\
0,& i\ne j.
\end{cases}
\]
Therefore
\begin{align*}
\E_W[I(\psi)]&=0,\\
\E_W[I(\psi)^2]
&=\sum_{i,j=0}^{m-1}c_i c_j
  \E_W[\Delta W_i\Delta W_j]
=\sum_{j=0}^{m-1}c_j^2(u_{j+1}-u_j)
=\|\psi\|_2^2.
\end{align*}
Since $I(\psi)$ is a linear combination of independent Gaussian
variables,
\[
I(\psi)\sim\N(0,\|\psi\|_2^2).
\]
Expressing two step functions on a common partition also gives
$I(\psi)-I(\chi)=I(\psi-\chi)$, hence
\[
\E_W[|I(\psi)-I(\chi)|^2]=\|\psi-\chi\|_2^2.
\]

\subsection*{Step 2: general deterministic integrands}

By density of step functions in $L^2([a,b])$, choose $\phi_n$ such that
\[
\|\phi_n-\phi\|_2\xrightarrow[n\to\infty]{}0,
\]
and set $\eta_n:=I(\phi_n)$. Step~1 gives
\[
\E_W[|\eta_n-\eta_m|^2]
=\|\phi_n-\phi_m\|_2^2
\xrightarrow[n,m\to\infty]{}0.
\]
Thus $(\eta_n)$ is Cauchy in $L^2(\Omega)$. By completeness, there is
$\eta\in L^2(\Omega)$ such that
\[
\E_W[|\eta_n-\eta|^2]\xrightarrow[n\to\infty]{}0.
\]
We define
\[
\int_a^b\phi(\tau)\,\dd W_\tau:=\eta.
\]
This definition is independent of the approximation: if
$\psi_n\to\phi$ in $L^2([a,b])$, then
\begin{align*}
\bigl(\E_W[|I(\phi_n)-I(\psi_n)|^2]\bigr)^{1/2}
&=\|\phi_n-\psi_n\|_2
\le \|\phi_n-\phi\|_2+\|\psi_n-\phi\|_2
\xrightarrow[n\to\infty]{}0,
\end{align*}
so both approximations have the same $L^2(\Omega)$ limit.

We next pass the mean and second moment to the limit. Since
$\E_W[\eta_n]=0$, Cauchy--Schwarz gives
\[
|\E_W[\eta]|
=|\E_W[\eta-\eta_n]|
\le \bigl(\E_W[|\eta-\eta_n|^2]\bigr)^{1/2}
\xrightarrow[n\to\infty]{}0.
\]
Hence $\E_W[\eta]=0$.

Set $V_n:=\|\phi_n\|_2^2$. The reverse triangle inequality gives
\[
|\sqrt{V_n}-\sqrt V|
=\bigl|\|\phi_n\|_2-\|\phi\|_2\bigr|
\le\|\phi_n-\phi\|_2
\xrightarrow[n\to\infty]{}0,
\]
so $V_n\to V$. Applying the same inequality in $L^2(\Omega)$,
and using $\E_W[\eta_n^2]=V_n$, yields
\[
\left|\sqrt{\E_W[\eta^2]}-\sqrt{V_n}\right|
\le \bigl(\E_W[|\eta-\eta_n|^2]\bigr)^{1/2}
\xrightarrow[n\to\infty]{}0.
\]
Consequently,
\[
\E_W[\eta^2]
=\lim_{n\to\infty}\E_W[\eta_n^2]
=\lim_{n\to\infty}V_n
=V.
\]
This proves the deterministic-integrand It\^o isometry:
\[
\boxed{
\E_W\!\left[
\left(\int_a^b\phi(\tau)\,\dd W_\tau\right)^2
\right]
=\int_a^b\phi(\tau)^2\,\dd\tau.
}
\]

To establish Gaussianity, fix $t\in\R$. Using
$|e^{\mathrm{i}x}-e^{\mathrm{i}y}|\le |x-y|$,
\begin{align*}
\left|
\E_W[e^{\mathrm{i}t\eta_n}]
-\E_W[e^{\mathrm{i}t\eta}]
\right|
&\le |t|\,\E_W[|\eta_n-\eta|]
\le |t|\bigl(\E_W[|\eta_n-\eta|^2]\bigr)^{1/2}
\xrightarrow[n\to\infty]{}0.
\end{align*}
Since $\eta_n\sim\N(0,V_n)$,
\[
\E_W[e^{\mathrm{i}t\eta}]
=\lim_{n\to\infty}\exp\!\left(-\frac{t^2V_n}{2}\right)
=\exp\!\left(-\frac{t^2V}{2}\right).
\]
This is the characteristic function of $\N(0,V)$, proving
\[
\boxed{
\int_a^b\phi(\tau)\,\dd W_\tau
\sim
\N\!\left(0,\int_a^b\phi(\tau)^2\,\dd\tau\right).
}
\]

\subsection*{Step 3: the vector-valued case}

Let $W=(W^{(1)},\ldots,W^{(d)})^\top$ be a standard
$d$-dimensional Brownian motion, whose components are independent.
Define the integral componentwise:
\[
\eta:=\int_a^b\phi(\tau)\,\dd W_\tau
=
\left(
\int_a^b\phi(\tau)\,\dd W_\tau^{(k)}
\right)_{k=1}^d.
\]
Using the same step approximations $\phi_n$ in every component gives
\[
\eta_n:=\int_a^b\phi_n(\tau)\,\dd W_\tau
\sim\N(0,V_n I_d),
\qquad
\E_W[\|\eta_n-\eta\|^2]
=\sum_{k=1}^d\E_W[|\eta_n^{(k)}-\eta^{(k)}|^2]
\xrightarrow[n\to\infty]{}0.
\]
For every $\theta\in\R^d$, the preceding characteristic-function
estimate therefore gives
\[
\E_W[e^{\mathrm{i}\theta^\top\eta}]
=\lim_{n\to\infty}
  \exp\!\left(-\frac{V_n\|\theta\|^2}{2}\right)
=\exp\!\left(-\frac{V\|\theta\|^2}{2}\right).
\]
Thus
\[
\boxed{\eta\sim\N(0,VI_d),}
\qquad
\E_W[\eta]=0,
\qquad
\E_W[\eta\eta^\top]=VI_d,
\]
where $V=\int_a^b\phi(\tau)^2\,\dd\tau$.

\section{Fokker--Planck Equation}
\label{app:fokker-planck}

Let $X_t$ be a continuous adapted solution of
\begin{equation}
\dd X_t=b_t(X_t)\,\dd t+g_t\,\dd W_t,
\qquad 0\le t\le T,
\label{eq:appendix-generic-sde}
\end{equation}
where $b:[0,T]\times\R^d\to\R^d$ and
$g:[0,T]\to\R$ are continuous, $g$ is deterministic, and
$W_t$ is a standard $d$-dimensional Brownian motion.
All processes use a common filtration $(\mathcal F_t)$:
adaptedness means measurability with respect to the information
$\mathcal F_t$ available at time $t$, and future Brownian
increments are independent of that information.

All spatial integrals are over $\R^d$. Write
\[
\mathcal L_t f=b_t\cdot\grad f+\frac12g_t^2\lap f,
\]
where $\mathcal L_t$ acts only on the spatial variable, including
when $f=f(t,x)$. The symbols $\grad$, $\diver$, and $\lap$
denote the spatial gradient, divergence, and Laplacian;
$\E$ denotes expectation and $\|\cdot\|_\infty$ the supremum
norm. We use $C_c^\infty$ for smooth compactly supported
functions and $C^{1,2}$ for functions with one continuous time
derivative and continuous spatial derivatives through order two.

\subsection{From the SDE to the Fokker--Planck equation}

\begin{lemma}[It\^o's formula]
\label{lem:fp-ito}
For $f\in C^{1,2}([0,T]\times\R^d)$ and
$0\le s\le t\le T$, almost surely,
\begin{align*}
f(t,X_t)-f(s,X_s)
&=\int_s^t(\partial_r f+\mathcal L_r f)(r,X_r)\,\dd r
+\int_s^t g_r\grad f(r,X_r)\cdot\dd W_r.
\end{align*}
For $f=f(x)$, the term $\partial_r f$ is absent.
\end{lemma}

The term $\frac12g_r^2\lap f$ accounts for Brownian increment
variance. The vector stochastic integral is interpreted
componentwise:
\[
\int_s^t H_r\cdot\dd W_r
=\sum_{i=1}^d\int_s^t H_r^{(i)}\,\dd W_r^{(i)}.
\]

\begin{lemma}[Zero expectation and It\^o isometry]
\label{lem:fp-zero-mean}
Let $H_t$ be a continuous adapted $\R^d$-valued process with
$\E\int_s^t|H_r|^2\,\dd r<\infty$. Then
\[
\E\!\left[\int_s^t H_r\cdot\dd W_r\right]=0,
\qquad
\E\!\left[\left|\int_s^t H_r\cdot\dd W_r\right|^2\right]
=\E\int_s^t|H_r|^2\,\dd r.
\]
\end{lemma}

\begin{proof}
First consider an adapted step process on a partition
$s=t_0<\cdots<t_m=t$: on $[t_j,t_{j+1})$, its value is
$H_r=H_{t_j}$, where $H_{t_j}$ is square-integrable and
$\mathcal F_{t_j}$-measurable. Its integral is
\[
\sum_{j=0}^{m-1}H_{t_j}\cdot(W_{t_{j+1}}-W_{t_j}).
\]
Each summand has mean zero and second moment
$(t_{j+1}-t_j)\E|H_{t_j}|^2$; distinct summands have zero
cross expectation. This proves both identities for step processes.
For continuous $H$, truncate first and then approximate by
adapted step processes in mean square over time and probability.
The isometry defines the general integral as their mean-square
limit; both identities pass to that limit.
\end{proof}

\begin{lemma}[Integration by parts]
\label{lem:fp-integration-by-parts}
For $f\in C_c^\infty(\R^d)$, a continuously differentiable
vector field $v$, and $p\in C^2(\R^d)$,
\[
\int v\cdot\grad f\,\dd x=-\int f\diver v\,\dd x,
\qquad
\int p\lap f\,\dd x=\int f\lap p\,\dd x.
\]
\end{lemma}

\begin{proof}
Integrate once in each coordinate for the first identity and
twice for the second. Compact support removes all boundary terms.
\end{proof}

\begin{theorem}
\label{thm:fp-forward}
If $X_t$ has density $p_t$ for every $t\in[0,T]$, then
\begin{equation}
\int f(x)\bigl(p_t(x)-p_s(x)\bigr)\,\dd x
=\int_s^t\int\mathcal L_r f(x)p_r(x)\,\dd x\,\dd r
\label{eq:fp-weak-density-path}
\end{equation}
for every $f\in C_c^\infty(\R^d)$ and $0\le s\le t\le T$.
If $b$ has continuous first spatial derivatives and
$p\in C^{1,2}((0,T)\times\R^d)$, this gives the classical
Fokker--Planck equation
\begin{equation}
\partial_t p_t=-\diver(b_tp_t)+\frac12g_t^2\lap p_t,
\qquad 0<t<T.
\label{eq:appendix-fp}
\end{equation}
\end{theorem}

\begin{proof}
For $f\in C_c^\infty(\R^d)$,
\[
\E\int_s^t|g_r\grad f(X_r)|^2\,\dd r
\le\|\grad f\|_\infty^2\int_s^t g_r^2\,\dd r<\infty.
\]
Lemmas~\ref{lem:fp-ito} and~\ref{lem:fp-zero-mean} give
$\E[f(X_t)]-\E[f(X_s)]=\int_s^t\E[\mathcal L_r f(X_r)]\,\dd r$.
Substituting the densities proves \eqref{eq:fp-weak-density-path};
compact support and continuity justify Fubini's theorem.
Under the additional regularity, differentiate in $t$ and use
Lemma~\ref{lem:fp-integration-by-parts}:
\[
\int f\,\partial_t p_t\,\dd x
=\int f\left[-\diver(b_tp_t)+\frac12g_t^2\lap p_t\right]\dd x.
\]
Arbitrary $f\in C_c^\infty$ and continuity yield
\eqref{eq:appendix-fp}.
\end{proof}

\subsection{From the Fokker--Planck equation to the marginal densities}

Weak continuity of $p_t(x)\,\dd x$ means continuity of
$t\mapsto\int f(x)p_t(x)\,\dd x$ for every bounded continuous
$f$. For a fixed spatial test $f$, we write $u$ for its
time-dependent backward extension.

\begin{lemma}[Smooth approximation of the drift]
\label{lem:fp-smooth-drift}
If $b$ is bounded and $|b_t(x)-b_t(y)|\le L|x-y|$ for all
$t,x,y$, there are drifts $b^{(n)}$, $n=1,2,\ldots$, with bounded
continuous spatial derivatives
of every order, the same Lipschitz constant, and
\[
\|b^{(n)}\|_\infty\le\|b\|_\infty,
\qquad \|b^{(n)}-b\|_\infty\le L/n.
\]
\end{lemma}

\begin{proof}
Convolve $b_t$ in space with a nonnegative smooth kernel of
integral one supported in the ball of radius $1/n$. Averaging
preserves the bound and Lipschitz constant, and changes $b_t(x)$
by at most $L/n$.
\end{proof}

\begin{lemma}[Backward test functions]
\label{lem:fp-backward-test}
Suppose $b$ is bounded and uniformly $L$-Lipschitz in $x$,
with bounded continuous first and second spatial derivatives.
For $0<\tau\le T$ and $f\in C_c^\infty(\R^d)$, there is
$u\in C^{1,2}([0,\tau]\times\R^d)$ satisfying
\begin{equation}
\partial_su+\mathcal L_su=0,
\qquad u(\tau,x)=f(x),
\label{eq:fp-backward-test}
\end{equation}
with $u$ and its $C^{1,2}$ derivatives bounded, and
\begin{equation}
\|u\|_\infty\le\|f\|_\infty,
\qquad \|\grad u\|_\infty\le e^{L\tau}\|\grad f\|_\infty.
\label{eq:fp-backward-bounds}
\end{equation}
Time derivatives at endpoints are one-sided.
\end{lemma}

\begin{proof}
Let $Y_r^{s,x}$ solve the same SDE starting from $x$ at time
$s$, and set $u(s,x)=\E[f(Y_\tau^{s,x})]$. Subtracting the
additive-noise path gives a globally Lipschitz ordinary integral
equation. Picard iteration gives its unique solution, and
Gronwall's inequality gives
\[
|Y_r^{s,x}-Y_r^{s,y}|\le e^{L(r-s)}|x-y|
\]
when both solutions use the same Brownian path. This proves
the gradient bound. Differentiating the integral equation once
and twice, then applying Gronwall, gives bounded continuous
spatial derivatives and justifies differentiation under expectation.

For $0<h\le\tau-s$, uniqueness and independent Brownian
increments give $u(s,x)=\E[u(s+h,Y_{s+h}^{s,x})]$. By
Lemmas~\ref{lem:fp-ito} and~\ref{lem:fp-zero-mean},
\[
\frac{u(s+h,x)-u(s,x)}{h}
=-\frac1h\int_s^{s+h}
\E[\mathcal L_r u(s+h,Y_r^{s,x})]\,\dd r.
\]
Continuity and boundedness give the limit $-\mathcal L_su(s,x)$.
The same argument on $[s-h,s]$ gives the left derivative.
Thus \eqref{eq:fp-backward-test} holds, and the backward
equation bounds the time derivative.
\end{proof}

\begin{lemma}[Time-dependent testing]
\label{lem:fp-time-testing}
Assume $b$ is bounded and $p_t$ are probability densities
satisfying \eqref{eq:fp-weak-density-path}, with weakly
continuous $p_t(x)\,\dd x$. For $f\in C^{1,2}$ with $f$
and its $C^{1,2}$ derivatives bounded,
\begin{equation}
\begin{aligned}
&\int f(t,x)p_t(x)\,\dd x-\int f(s,x)p_s(x)\,\dd x\\
&\qquad=\int_s^t\int(\partial_r f+\mathcal L_r f)(r,x)
p_r(x)\,\dd x\,\dd r.
\end{aligned}
\label{eq:fp-time-testing}
\end{equation}
\end{lemma}

\begin{proof}
For fixed compact spatial support, approximate spatially by
smooth functions and apply \eqref{eq:fp-weak-density-path}
on each time subinterval with $f$ frozen at its left endpoint.
Summing and sending the mesh to zero gives
\eqref{eq:fp-time-testing}; the change in $f$ supplies
$\partial_r f$, and weak continuity justifies the limit.

For general $f$ and $R\ge1$, choose smooth cutoffs $\chi_R$ equal to one
on $|x|\le R$, zero on $|x|\ge2R$, and satisfying
$0\le\chi_R\le1$, $\|\grad\chi_R\|_\infty=O(R^{-1})$,
and $\|\lap\chi_R\|_\infty=O(R^{-2})$. Testing with
$\chi_Rf$ adds only
\[
f b_r\cdot\grad\chi_R
+g_r^2\grad\chi_R\cdot\grad f
+\frac12g_r^2f\lap\chi_R
\]
to the term $\chi_R(\partial_r+\mathcal L_r)f$.
Their integrals vanish as $R\to\infty$ because $\int p_r=1$.
Dominated convergence gives \eqref{eq:fp-time-testing}.
\end{proof}

\begin{theorem}
\label{thm:fp-marginal-identification}
Assume $b$ is bounded and uniformly $L$-Lipschitz in $x$.
Let $p_t$ be probability densities satisfying
\eqref{eq:fp-weak-density-path}, with weakly continuous
$p_t(x)\,\dd x$. If $X_0$ has density $p_0$, then $X_t$ has
density $p_t$ for every $t\in[0,T]$:
\begin{equation}
\mathbb P(X_t\in A)=\int_A p_t(x)\,\dd x
\quad\text{for every Borel set }A\subseteq\R^d.
\label{eq:fp-identified-density}
\end{equation}
\end{theorem}

A classical solution of \eqref{eq:appendix-fp} with the stated
weak continuity satisfies \eqref{eq:fp-weak-density-path} by
integration by parts and time integration.

\begin{proof}
Fix $0<\tau\le T$ and $f\in C_c^\infty(\R^d)$. Apply
Lemma~\ref{lem:fp-backward-test} to the smooth drifts
$b^{(n)}$ from Lemma~\ref{lem:fp-smooth-drift}, obtaining
$u_n$ with terminal value $f$. Then
\[
(\partial_s+\mathcal L_s)u_n
=(b_s-b_s^{(n)})\cdot\grad u_n.
\]
Lemma~\ref{lem:fp-time-testing} and It\^o's formula give
\begin{align*}
\int f p_\tau\,\dd x-\int u_n(0,x)p_0(x)\,\dd x
&=\int_0^\tau\int
(\partial_su_n+\mathcal L_su_n)(s,x)p_s(x)\,\dd x\,\dd s,\\
\E[f(X_\tau)]-\int u_n(0,x)p_0(x)\,\dd x
&=\int_0^\tau
\E[(\partial_su_n+\mathcal L_su_n)(s,X_s)]\,\dd s.
\end{align*}
The second identity uses the initial density $p_0$ and
Lemma~\ref{lem:fp-zero-mean}, since
\[
\E\int_0^\tau|g_s\grad u_n(s,X_s)|^2\,\dd s
\le e^{2L\tau}\|\grad f\|_\infty^2
\int_0^\tau g_s^2\,\dd s<\infty.
\]
Subtracting and using \eqref{eq:fp-backward-bounds} gives
\[
\left|\E[f(X_\tau)]-\int f p_\tau\,\dd x\right|
\le2\tau\|b-b^{(n)}\|_\infty e^{L\tau}\|\grad f\|_\infty
\longrightarrow0.
\]
Since $C_c^\infty$ test functions determine finite Borel
measures, $X_\tau$ has density $p_\tau$. The case $t=0$ is
the initial assumption.
\end{proof}

\section{Conditional-to-Marginal Averaging Lemma}
\label{app:conditional-to-marginal}

\begin{lemma}[Averaging Fokker--Planck equations]
Suppose that for each fixed $x_0$, a conditional density $p_t(x \mid x_0)$ satisfies
\[
\partial_t p_t(x \mid x_0)
= -\diver\bigl(p_t(x \mid x_0)b_t(x \mid x_0)\bigr)
+ \frac{1}{2}g_t^2 \lap p_t(x \mid x_0).
\]
Then the marginal density
\[
p_t(x)=\int p_t(x \mid x_0)p_0(x_0)\dd x_0
\]
satisfies
\[
\partial_t p_t(x)
= -\diver\bigl(p_t(x)\bar b_t(x)\bigr)
+ \frac{1}{2}g_t^2 \lap p_t(x),
\]
where
\begin{equation}
\bar b_t(x)
:=
\E_{X_0\mid X_t=x}[b_t(x \mid X_0)].
\label{eq:appendix-marginal-drift}
\end{equation}
\end{lemma}

\begin{proof}
Differentiate under the integral sign:
\[
\partial_t p_t(x)
= \int \partial_t p_t(x \mid x_0)p_0(x_0)\dd x_0.
\]
Substitute the conditional Fokker--Planck equation:
\begin{align*}
\partial_t p_t(x)
&= -\int \diver\bigl(p_t(x \mid x_0)b_t(x \mid x_0)\bigr)p_0(x_0)\dd x_0\\
&\quad + \frac{1}{2}g_t^2\int \lap p_t(x \mid x_0)p_0(x_0)\dd x_0.
\end{align*}
Move derivatives outside the integrals:
\[
\partial_t p_t(x)
= -\diver\left(
\int p_t(x \mid x_0)b_t(x \mid x_0)p_0(x_0)\dd x_0
\right)
+ \frac{1}{2}g_t^2 \lap p_t(x).
\]
Define
\[
\bar b_t(x)
:=
\frac{\int p_t(x \mid x_0)b_t(x \mid x_0)p_0(x_0)\dd x_0}{p_t(x)}.
\]
By Bayes' rule this is \eqref{eq:appendix-marginal-drift}, so
\[
\partial_t p_t(x)
= -\diver\bigl(p_t(x)\bar b_t(x)\bigr)
+ \frac{1}{2}g_t^2 \lap p_t(x).
\qedhere
\]
\end{proof}

\section{Fisher's Identity}
\label{app:fisher}

\begin{theorem}[Fisher's identity]
If
\[
p_t(x)=\int p_t(x \mid x_0)p_0(x_0)\dd x_0,
\]
then
\begin{equation}
\grad\log p_t(x)
= \E_{X_0}[\grad\log p_t(x \mid X_0)\mid X_t=x].
\label{eq:appendix-fisher}
\end{equation}
\end{theorem}

\begin{proof}
Differentiate the marginal density:
\[
\grad p_t(x)
= \int \grad p_t(x \mid x_0)p_0(x_0)\dd x_0.
\]
Divide by $p_t(x)$:
\[
\grad\log p_t(x)
= \int \frac{\grad p_t(x \mid x_0)}{p_t(x)}p_0(x_0)\dd x_0.
\]
Multiply and divide by $p_t(x \mid x_0)$ inside the integral:
\[
\grad\log p_t(x)
= \int
\frac{\grad p_t(x \mid x_0)}{p_t(x \mid x_0)}
\frac{p_t(x \mid x_0)p_0(x_0)}{p_t(x)}
\dd x_0.
\]
The first factor is $\grad\log p_t(x \mid x_0)$. The second factor is $p(x_0 \mid X_t=x)$ by Bayes' rule. Hence
\[
\grad\log p_t(x)
= \int \grad\log p_t(x \mid x_0)\,p(x_0 \mid X_t=x)\dd x_0,
\]
which is \eqref{eq:appendix-fisher}.
\end{proof}

\section{Orthogonality Identity Used in the Denoising Loss}
\label{app:orthogonality}

\begin{theorem}[Orthogonality identity]
Let $U \in \mathbb{R}^p$ and $V \in \mathbb{R}^d$ be jointly
distributed random vectors, representing the input and target,
respectively. Let $f:\mathbb{R}^p \to \mathbb{R}^d$ be a
measurable prediction function, and assume that $V$ and $f(U)$
have finite second moments. Then
\begin{equation}
\mathbb{E}_{U,V}\bigl[\|f(U)-V\|_2^2\bigr]
=
\mathbb{E}_U\bigl[\|f(U)-\mathbb{E}[V\mid U]\|_2^2\bigr]
+
\mathbb{E}_{U,V}\bigl[\|V-\mathbb{E}[V\mid U]\|_2^2\bigr].
\label{eq:appendix-orthogonality}
\end{equation}
Here $\mathbb{E}_{U,V}$ denotes expectation with respect to
the joint distribution of $(U,V)$, $\mathbb{E}_U$ denotes
expectation with respect to the marginal distribution of $U$,
and $\mathbb{E}[V\mid U]$ denotes expectation of $V$ with
respect to its conditional distribution given $U$.
\end{theorem}

\begin{proof}
Let $m(U):=\mathbb{E}[V\mid U]$. Then
\[
f(U)-V
=
\bigl(f(U)-m(U)\bigr)-\bigl(V-m(U)\bigr).
\]
Expanding the squared norm gives
\[
\begin{aligned}
\|f(U)-V\|_2^2
&=
\|f(U)-m(U)\|_2^2
+
\|V-m(U)\|_2^2 \\
&\quad
-2\bigl(f(U)-m(U)\bigr)^\top\bigl(V-m(U)\bigr).
\end{aligned}
\]
Since $f(U)-m(U)$ is measurable with respect to $U$,
the tower property yields
\[
\begin{aligned}
&\mathbb{E}_{U,V}\Bigl[
\bigl(f(U)-m(U)\bigr)^\top\bigl(V-m(U)\bigr)
\Bigr] =
\mathbb{E}_U\Bigl[
\bigl(f(U)-m(U)\bigr)^\top
\mathbb{E}[V-m(U)\mid U]
\Bigr]
=0,
\end{aligned}
\]
where the last equality follows from
\[
\mathbb{E}[V-m(U)\mid U]
=
\mathbb{E}[V\mid U]-m(U)
=0.
\]
Taking expectations in the squared-norm expansion therefore gives
\[
\mathbb{E}_{U,V}\bigl[\|f(U)-V\|_2^2\bigr]
=
\mathbb{E}_U\bigl[\|f(U)-m(U)\|_2^2\bigr]
+
\mathbb{E}_{U,V}\bigl[\|V-m(U)\|_2^2\bigr].
\]
Substituting $m(U)=\mathbb{E}[V\mid U]$ proves the identity.
\end{proof}

\newpage
\phantomsection
\addcontentsline{toc}{section}{References}
\nocite{*}
\bibliographystyle{icml2027}
\begingroup
\renewcommand{\ttdefault}{tutorialmono}
\urlstyle{tt}
\interlinepenalty=10000
\bibliography{references}
\endgroup

\end{document}